\documentclass[twoside]{article}

\usepackage[accepted]{aistats2025}
\usepackage{enumitem}
\usepackage{amsmath}
\usepackage{amssymb}
\usepackage{mathtools}
\usepackage{amsthm}
\usepackage{dsfont}
\usepackage{subfigure}

\theoremstyle{plain}
\newtheorem{theorem}{Theorem}[section]
\newtheorem{proposition}[theorem]{Proposition}
\newtheorem{lemma}[theorem]{Lemma}

\theoremstyle{definition}

\newtheorem{assumption}[theorem]{Assumption}
\theoremstyle{remark}
\newtheorem{remark}[theorem]{Remark}

\begin{document}

%

%
\runningtitle{Quantifying GNN vs. MLP}
\runningauthor{Huang, Cao, Wang, Cao, Suzuki}

\twocolumn[

\aistatstitle{Quantifying the Optimization and Generalization Advantages of Graph Neural Networks Over Multilayer Perceptrons}



\aistatsauthor{%
  Wei Huang\textsuperscript{1} \And
  Yuan Cao\textsuperscript{2} \And
  Haonan Wang\textsuperscript{3} \And
  Xin Cao\textsuperscript{4} \And
  Taiji Suzuki\textsuperscript{5,1}
}
\aistatsaddress{%
  \textsuperscript{1}RIKEN AIP 
  \textsuperscript{2}The University of Hong Kong 
  \textsuperscript{3}National University of Singapore \\
  \textsuperscript{4}The University of New South Wales 
  \textsuperscript{5}University of Tokyo
}]

\begin{abstract}
 Graph neural networks (GNNs) have demonstrated remarkable capabilities in learning from graph-structured data, often outperforming traditional Multilayer Perceptrons (MLPs) in numerous graph-based tasks. Although existing works have demonstrated the benefits of graph convolution through Laplacian smoothing, expressivity or separability, there remains a lack of quantitative analysis comparing GNNs and MLPs from an optimization and generalization perspective. This study aims to address this gap by examining the role of graph convolution through feature learning theory. Using a signal-noise data model, we conduct a comparative analysis of the optimization and generalization between two-layer graph convolutional networks (GCNs) and their MLP counterparts. Our approach tracks the trajectory of signal learning and noise memorization in GNNs, characterizing their post-training generalization. We reveal that GNNs significantly prioritize signal learning, thus enhancing the regime of {low test error} over MLPs by $D^{q-2}$ times, where $D$ denotes a node's expected degree and $q$ is the power of ReLU activation function with $q>2$. This finding highlights a substantial and quantitative discrepancy between GNNs and MLPs in terms of optimization and generalization, a conclusion further supported by our empirical simulations on both synthetic and real-world datasets.
\end{abstract}

\section{Introduction}

Graph neural networks (GNNs) have recently demonstrated remarkable capability in learning graph representations, yielding superior results across various downstream tasks, such as node classifications~\cite{kipf2016semi, velivckovic2017graph, hamilton2017inductive}, graph classifications~\cite{xu2018powerful,gilmer2017neural,lee2019self,Yuan2020StructPoolSG} and link predictions~\cite{kipf2016variational,zhang2018link,kumar2020link}, etc. Compared to multilayer perceptron (MLPs), GNNs enhance representation learning with an added message passing operation \cite{zhou2020graph}. Take graph convoluational network (GCN) \cite{kipf2016semi} as an example, it aggregates a node's attributes with those of its neighbors through a \textit{graph convolution} operation. This operation, which leverages the structural information (adjacency matrix) of graph data, forms the core distinction between GNNs and MLPs.

There is substantial evidence showing that GCNs consistently outperform MLPs empirically \cite{kipf2016semi,rong2019dropedge,zhou2020graph,shchur2018pitfalls,chen2017supervised,dwivedi2023benchmarking,ma2021homophily}. This success has led to numerous  theoretical studies aimed at understanding the underlying reasons for the effectiveness of graph convolutions. One prominent explanation is that graph convolutions leverage Laplacian smoothing to filter out noise from input features, as highlighted by several works~\cite{li2018deeper,chen2020measuring,nt2019revisiting,oono2019graph}. Another line of research investigates the expressivity of graph neural networks in distinguishing complex graph structures \cite{garg2020generalization,loukas2020hard,xu2018powerful}. Recent studies have also explored the theoretical role of graph convolutions in enhancing separability. For instance, \cite{baranwal2021graph} considered a setting of linear classification of data generated from a contextual stochastic block model \cite{deshpande2018contextual}, showing that graph convolution extends the regime where data is linearly separable by a factor of approximately $1/\sqrt{D}$ compared to MLPs, with $D$ denoting a node's expected degree. Building on this, \cite{baranwal2023effects} further examined multi-layer graph nerual networks and demonstrated improved non-linear separability with the incorporation of graph convolutions.

Despite the valuable insights provided by the existing literature, there has been limited theoretical research that offers a \textit{quantitative} understanding of the optimization and generalization properties of GNNs compared to their MLP counterparts in a unified framework. To address this gap, we aim to answer the following key question from a theoretical perspective:

  \textit{What role does graph convolution play during gradient descent training, and to what extent do GCNs exhibit better generalization compared to MLPs?}

To address this critical question, we conduct a feature learning analysis \cite{cao2022benign,allen2022feature} for graph neural networks to establish a unified theoretical framework that analyzes both the convergence and generalization properties of GCNs, enabling a \textit{quantitative comparison} with MLPs. Specifically, we introduce a data generation model, termed, that combines a signal-noise model \cite{allen2020towards,cao2022benign} for input feature creation with a stochastic block model \cite{abbe2015exact} for graph construction. This setting serves as a representative case study for our analysis. Our theoretical investigation focuses on the optimization trajectory and post-training generalization of two-layer GCNs trained using gradient descent, and compares these results with the established findings for two-layer MLPs \cite{cao2022benign}. While both GCNs and MLPs are shown to achieve near-zero training error, our analysis reveals a distinct quantitative advantage for GCNs in terms of generalization on test data. The key contributions of our study are as follows:

\begin{itemize} [leftmargin = *]
    \item  \textbf{Global Convergence Guarantees:} We provide global convergence guarantees for graph neural networks trained on data generated by the SNM-SBM model. By characterizing signal learning and noise memorization in feature learning, we show that, despite the inherent nonconvexity of the optimization landscape, GCNs can achieve zero training error within a polynomial number of iterations.
    
    \item  \textbf{Test Error Bounds for Overfitted GNNs:} We derive theoretical test error bounds for overfitted GNN models trained using gradient descent. Under specific conditions on the signal-to-noise ratio, we demonstrate that GCNs can achieve small (near-zero) test error, even when the model is over-fitted to the noisy data.

    \item \textbf{Quantitative Generalization Comparison:} We provide a quantitative comparison of the generalization performance between GCNs and MLPs. Our results identify a regime where GCNs achieve nearly zero test error, while MLPs exhibit a significantly higher test error. This finding is further validated through empirical experiments.
\end{itemize}

\section{Related Work}
 
\paragraph{Role of Graph Convolution in GNNs.} Existing studies \cite{chen2017supervised,ma2021homophily,zhang2019graph,he2020lightgcn,wang2023snowflake} have shown that graph convolutions significantly enhance the performance of traditional classification methods, such as MLPs. Motivated by these benefits, various graph network architectures and methods have been proposed to further harness the power of graph convolutions \cite{wu2019simplifying,gasteiger2018predict,wu2022towards}. From a theoretical standpoint, \cite{xu2020neural} highlight the superiority of GNNs for extrapolation problems in comparison with MLPs, based on the graph neural tangent kernel (GNTK)  \cite{jacot2018neural,du2019graph,huang2021towards,sabanayagam2022representation}. \cite{huang2021towards} use a similar approach to examine the role of graph convolution in deep GNNs, revealing that excessive graph layers can degrade optimization and generalization, thus supporting the well-known over-smoothing problem in deep GNNs \cite{li2018deeper}. In addition, \cite{yang2022graph} attribute the major performance gains of GNNs to their inherent generalization capability. \cite{chen2020graph} primarily examine the expressivity of Graph-Augmented MLPs. Similarly, \cite{hou2022measuring} propose two smoothness metrics to quantify the quality of information extracted from graph data and introduce a novel attention-based framework to optimize GNN performance. Besides, \cite{xu2021optimization} investigate the impact of skip connections, increased depth, and favorable label distributions on the training dynamics of GNNs. In contrast to these existing theoretical results, our work moves beyond GNTK analysis and differs from common approaches like Laplacian smoothing, expressivity, or separability studies. Instead, we focus on developing a unified framework to quantitatively compare GNNs and MLPs through a comprehensive convergence and generalization analysis.

\paragraph{Feature Learning in Deep Learning.} This work builds upon a growing body of research on how neural networks learn features. \cite{allen2020towards} formulate a theory illustrating that when data possess a ``multi-view'' feature, ensembles of independently trained neural networks can demonstrably improve test accuracy. Further, \cite{allen2022feature} show that adversarial training can eliminate specific small dense mixtures from the hidden layers, thereby refining the learned weights. Additionally, \cite{ba2022high} identify that the initial gradient update introduces a rank-1 `spike', waligning the first-layer weights with the linear component of the teacher model's features. The seminal work by \cite{cao2022benign} investigates the benign overfitting phenomenon in two-layer convolutional neural networks (CNNs) and reveals that under certain signal-to-noise ratio conditions, gradient descent can drive a two-layer CNN to achieve near-zero test loss through effective feature learning. Related works \cite{yang2020feature,zou2021understanding,wen2021toward,damian2022neural,zou2023benefits,chen2023towards,meng2023per,jelassi2022vision,kou2023benign,chen2023does,li2023theoretical,huang2025comparison,han2024feature,jiang2025unveil,bu2025provably,huang2023understanding} have similarly highlighted the role of feature learning in neural networks during gradient descent training, forming a critical area of research that our study continues to explore. While \cite{li2023improves,zhang2023joint} apply feature learning theory to analyze GNNs, they do not quantify the differences in optimization and generalization between GNNs and MLPs. Our work addresses this gap by providing a \textit{quantitative comparison}, offering new insights into the distinct behaviors of these architectures. 

\section{Problem Setup} \label{sec:problem_setup}

\subsection{Notations}

We use lower bold-faced letters for vectors, upper bold-faced letters for matrices, and non-bold-faced letters for scalars. For a vector $\mathbf{v} $, its $\ell_2$-norm is denoted as $\| \mathbf{v} \|_2 $. For a matrix $\mathbf{A}$, we use $\| \mathbf{A}\|_2$ to denote its spectral norm and $\| \mathbf{A}\|_F$ for its Frobenius norm. We employ standard asymptotic notations such as $O(\cdot)$, $o(\cdot)$, $\Omega(\cdot)$, and $\Theta(\cdot)$ to describe the limiting behavior. We use $\widetilde{O}(\cdot)$, $\widetilde{\Omega}(\cdot)$, and $\widetilde{\Theta}(\cdot)$ to hide logarithmic factors in these notations respectively. Moreover, we denote $a_n = \textrm{poly} (b_n)$ if $a_n = O((b_n)^p)$ for some positive constant $p$ and $a_n = \textrm{polylog}(b_n)$ if $a_n = \textrm{poly}( \log(b_n))$. 
Lastly, sequences of integers are denoted as $[m] = \{1,2,\ldots,m \}$.

\subsection{Data Model}

We adopt a combined \textit{signal-noise model} (SNM) for feature generation and a \textit{stochastic block model} for graph structure generation as a case study. 

\paragraph{Signal-noise Model} Specifically, let the feature matrix be denoted as $\mathbf{X} \in \mathbb{R}^{n \times 2d}$, where $n$ represents the number of samples and $2d$ is the feature dimensionality. Each feature associated with a data point is generated from a SNM, conditional on the Rademacher random variable $y \in \{-1, 1 \}$, and a latent vector $\boldsymbol{\mu} \in \mathbb{R}^d$:
\begin{equation} \label{eq:gaussian}
{\bf x} = [\mathbf{x}^{(1)}, \mathbf{x}^{(2)} ] = [y  \boldsymbol{\mu}, \boldsymbol{\xi} ],
\end{equation}

where $\mathbf{x}^{(1)}, \mathbf{x}^{(2)} \in \mathbb{R}^d$, and $\boldsymbol{\xi}  \sim \mathcal{N}(\mathbf{0}, \sigma_p^2 \cdot (\mathbf{I} -  \| \boldsymbol{\mu} \|_2^{-2} \cdot \boldsymbol{\mu} \boldsymbol{\mu}^\top))$ with $\sigma^2_p$ as the variance. The term $ \mathbf{I} -  \| \boldsymbol{\mu} \|_2^{-2} \cdot \boldsymbol{\mu} \boldsymbol{\mu}^\top $ ensures that the noise vector $\boldsymbol{\xi}$ is orthogonal to the signal vector $\boldsymbol{\mu}$. 

We make the following remarks on the data model.
\begin{itemize} [leftmargin = *]
    \item  \textbf{Feature Composition:} The data model simulates a setting where the input features consist of both signal and noise components. In particular, the term
    $y \boldsymbol{\mu}$ represents task (label)-relevant features, while $\boldsymbol{\xi}$ captures task-irrelevant features. Recent works \cite{allen2020towards,cao2022benign,zou2021understanding,shen2022data} have explored similar signal-noise models to investigate the feature learning process of neural networks, including studies focused on graph neural networks \cite{li2023improves,zhang2023joint}.

    \item \textbf{Real-world Relevance:} The SNM reflects the input feature structure of real-world graph datasets used in node classification tasks. For example, citation network datasets such as Cora, Citeseer, and Pubmed~\cite{kipf2016semi} typically use a bag-of-words representation for node features. Conceptually, these words can be divided into two categories: task-relevant and task-irrelevant. Words like ``algorithm" or ``neural network" are task-relevant for the field of computer science, whereas more generic words like ``study" or ``approach" are task-irrelevant. 
\end{itemize}

\paragraph{Stochstic Block Model} We employ a SBM to generate the graph structure, with intra-class edge probability $p$ and inter-class edge probability $s$. Specifically, each entry in the adjacency matrix $\mathbf{A} = (a_{ij})_{n \times n}$ follows a Bernoulli distribution: $a_{ij} \sim \textrm{Ber}(p)$ when $y_i = y_j$, and $a_{ij} \sim \textrm{Ber}(s)$ when $y_i \neq y_j$. The use of $p$ and $s$ are explicitly modeled, allows us to explicitly model different graph structures and analyze the impact of varying connectivity patterns. 

When $p \gg s$, the graph structure exhibits \textit{homophily}, meaning that nodes are more likely to connect with others that share the same label, resulting in clusters of similarly labeled nodes. Conversely, when $s \gg p$, the graph reflects a \textit{heterophily}, where nodes are more likely to connect to those with different labels. This flexibility makes SBM a powerful tool for simulating diverse graph topologies, and it is widely used in related studies on GNNs \cite{baranwal2021graph,mehta2019stochastic,baranwal2023effects,ma2021homophily,keriven2024functions}. 

We represent the combination of the SNM and SBM as $\mathrm{SNM-SBM}(n,p,s,\boldsymbol{\mu},\sigma_p,d)$. This combined model captures both feature and structural variations, providing a unified framework for studying GNNs. Note that when $p = s = 0$, the graph structure disappears, and $\mathrm{SNM-SBM}$ reduces to the standard SNM. 

\subsection{Graph Neural Network Model}

Graph neural network (GNNs) integrate both graph structure and node features to learn meaningful representations for nodes. Consider a two-layer GCN model, denoted as $f$, where the first layer performs a graph convolution operation, and the second layer parameters are fixed to either $+1$ or $-1$. The output of the GCN is given by:
\begin{align*}
    {f} (\mathbf{W}, \mathbf{A},{\mathbf{x}}) = F_{+1}(\mathbf{W}_{+1}, \mathbf{A},  {\mathbf{x}}) - {F}_{-1}(\mathbf{W}_{-1}, \mathbf{A}, {\mathbf{x}}),
\end{align*} 
where $ {F}_{+1}(\mathbf{W}_{+1}, \mathbf{A},  {\mathbf{x}})$ and $ {F}_{-1}(\mathbf{W}_{-1}, \mathbf{A},  {\mathbf{x}})$ are defined as:
\begin{equation} \label{eq:gcn}
     {F}_j(\mathbf{W}_j, \mathbf{A},  {\mathbf{x}} ) = \frac{1}{m}\sum_{r=1}^m \left[ \sigma(\mathbf{w}_{j,r}^\top \Tilde{\mathbf{x}}^{(1)}) + \sigma(\mathbf{w}_{j,r}^\top \Tilde{\mathbf{x}}^{(2)}) \right].
\end{equation}
where $j \in  \{+1, -1 \}$, and $\mathbf{W}_{\pm 1}$ 
 refer to the first layer weights associated with the second-layer fixed parameters. Besides, $\sigma(\cdot)$ is a polynomial ReLU activation function defined as $\sigma(z) = \max \{0,z\}^q$ for some $q > 2$, and $m$ is the width of hidden layer. The notation $\Tilde{\mathbf{X}} \triangleq [ \tilde{\mathbf{x}}_1, \tilde{\mathbf{x}}_2, \cdots, \tilde{\mathbf{x}}_n]^\top = \tilde{\mathbf{D}}^{-1} \tilde{\mathbf{A}} \mathbf{X}  \in \mathbb{R}^{n \times 2d}$ represents the node features after the graph convolution.

Specifically, the adjacency matrix with self-loops is defined as $\tilde{\mathbf{A}} = \mathbf{A} + \mathbf{I}_n$, where $\mathbf{A} $ is the original adjacency matrix and $ \mathbf{I}_n$ is the identity matrix of size $n$. The diagonal degree matrix $\tilde{\mathbf{D}}$
records the degree of each node, with entries given by $\tilde{D}_{ii} = \sum_j \tilde{A}_{ij}$. For simplicity we denote $ D_i \triangleq \tilde{D}_{ii} $. 

\begin{remark}
    The use of a polynomial ReLU activation function ensures a significant gap between signal learning and noise memorization. This type of activation function has been widely adopted in related works \cite{allen2020towards,allen2022feature,cao2022benign,zou2021understanding,kou2023benign} to facilitate the theoretical study of neural network training dynamics.
\end{remark}

Moreover, the symbol $\mathbf{W}$ collectively denotes the first layer's weights, and $\mathbf{w}_{j,r} \in \mathbb{R}^{d}$ refers to the weight of the first layer, in which $r$ corresponds to hidden neuron index and $j \in \{+1, -1 \}$ refer to the fixed value of second layer. The first-layer weights are initialized by sampling from a Gaussian distribution, i.e., $ \mathbf{w}_{j,r} \sim \mathcal{N}(\mathbf{0}, \sigma^2_0 \cdot \mathbf{I}_{d \times d})$ for all $r \in [m]$ and $j \in \{-1, 1 \}$, where $\sigma_0$ controls the strength of the initial weights.

Given the training data $\mathcal{S} \triangleq \{ \mathbf{x}_i,y_i \}_{i=1}^n$ and adjacency matrix $\mathbf{A} \in \mathbb{R}^{n \times n}$ drawn from $\mathrm{SNM-SBM}(n,p,s,\boldsymbol{\mu},\sigma_p,d)$, we aim to learn the parameter $\mathbf{W}$ by minimizing the empirical cross-entropy loss function:
\begin{equation}
    L^{\mathrm{GCN}}_{\mathcal{S}}(\mathbf{W}) = \frac{1}{n} \sum_{i=1}^n \ell(y_i \cdot  {f}(\mathbf{W}, \mathbf{A}, {\mathbf{x}}_i) ).
\end{equation}
Here, the cross-entropy loss is defined as $\ell(y \cdot  {f}(\mathbf{W}, \mathbf{A}, {\mathbf{x}}) ) = \log(1+ \exp(- f(\mathbf{W, \mathbf{A}, \mathbf{x}} ) \cdot  y ))$. The gradient descent update for the first layer weight $\mathbf{W}$ in GCN can be expressed as:
\begin{align}
    \mathbf{w}_{j,r}^{(t+1)} &= \mathbf{w}_{j,r}^{(t)} - \eta \cdot \nabla_{\mathbf{w}_{j,r}} L^{\mathrm{GCN}}_{\mathcal{S}}(\mathbf{W}^{(t)}) \nonumber  \\
    &= \mathbf{w}_{j,r}^{(t)} - \frac{\eta}{nm} \sum_{i=1}^n  {\ell}_i'^{(t)}    \sigma'(\langle \mathbf{w}_{j,r}^{(t)},   \tilde{y}_i  {\boldsymbol{\mu}} \rangle) \cdot j  \tilde{y}_i  {\boldsymbol{\mu}} \nonumber \\
    &\quad - \frac{\eta}{nm} \sum_{i=1}^n {\ell}_i'^{(t)}   \sigma'(\langle \mathbf{w}_{j,r}^{(t)},   \Tilde{\boldsymbol{\xi}}_i \rangle) 
    \cdot j y_i \Tilde{\boldsymbol{\xi}}_{i}  , 
    \label{eq:gcn_gdupdate} 
\end{align}
where we define the loss derivative as $ {\ell}'_i \triangleq \ell'(y_i \cdot {f}_i)  =-  \frac{ \exp(-y_i \cdot  {f}_i)}{1 + \exp(-y_i  \cdot  {f}_i)} $, the ``aggregated label'' $ \Tilde{y}_i = D_{i}^{-1} \sum_{k \in \mathcal{N}(i)} y_k  $, and the ``aggregated noise vector'' $\Tilde{\boldsymbol{\xi}}_i  = D_{i}^{-1} \sum_{k \in \mathcal{N}(i)}   \boldsymbol{\xi}_k$. Here $\mathcal{N}(i)$ denotes the set of neighbors of node $i$, $\sigma'(\cdot)$ represents the derivative of the polynomial ReLU activation function, and $\eta$ is the learning rate.

To quantify the learning capabilities of GNNs compared to MLPs, we analyze the generalization ability of GNN models through the lens of population loss, which is defined based on unseen test data. After training the network on $n$ data points, we generate a new test data point following the $\mathrm{SNM-SBM}$ distribution. Its connections to the training data points are determined using the stochastic block model, forming an updated adjacency matrix $\mathbf{A}' \in \mathbb{R}^{(n+1) \times (n+1)}$. The population loss is then calculated by taking the expectation over the randomness of the new test data, and is expressed as follows:
\begin{equation} \label{eq:population_gnn}
L^{\mathrm{GCN}}_{\mathcal{D}} (\mathbf{W}) = \mathbb{E}_{({\mathbf{x}}, y, \mathbf{A}') \sim \mathrm{SNM-SBM} }  \ell(y \cdot f(\mathbf{W}, \mathbf{A}', {\mathbf{x}})).
\end{equation}
This approach for formulating the generalization error is consistent with the methodology used in \cite{li2023improves}.

\subsection{Weight Decomposition for Optimization Analysis}

To track the complex training dynamics described by Equation (\ref{eq:gcn_gdupdate}), we employ a weight decomposition method inspired by feature learning theory \cite{cao2022benign}. From the gradient descent rule in Equation (\ref{eq:gcn_gdupdate}), we observe that each gradient descent iterate $\mathbf{w}_{j,r}^{(t)}$ can be represented as a linear combination of its initial random weight $\mathbf{w}_{j,r}^{(0)}$, the signal vector $\boldsymbol{\mu}$, and the noise vectors $\boldsymbol{\xi}_i$\footnote{By referring to Equation (\ref{eq:gcn_gdupdate}), we see that the gradient descent update moves in the direction of $\tilde{\boldsymbol{\xi}}_i$, which can be further decomposed into ${\boldsymbol{\xi}}_i$ through $\Tilde{\boldsymbol{\xi}}_i = D_{i}^{-1} \sum_{k \in \mathcal{N}(i)} \boldsymbol{\xi}_k$.} from the training data for $i \in [n]$. For $r \in [m]$, the weight decomposition at iteration $t$ can be expressed:
\begin{align}
    \mathbf{w}_{j,r}^{(t)} & = \mathbf{w}_{j,r}^{(0)} + j  \gamma_{j,r}^{(t)} \cdot \| \boldsymbol{\mu} \|_2^{-2} \cdot \boldsymbol{\mu} + \sum_{ i = 1}^n \overline{\rho}_{j,r,i}^{(t) }\cdot \| {\boldsymbol{\xi}}_i \|_2^{-2} \cdot  {\boldsymbol{\xi}}_{i}  \nonumber \\
    & \quad + \sum_{ i = 1}^n \underline{\rho}_{j,r,i}^{(t) }\cdot \|  {\boldsymbol{\xi}}_i \|_2^{-2} \cdot {\boldsymbol{\xi}}_{i} . \label{eq:w_decomposition}
\end{align}
where $ \gamma_{j,r}^{(t)} $ and $\rho_{j,r,i}^{(t)} = \{ \overline{\rho}_{j,r,i}^{(t)} , \underline{\rho}_{j,r,i}^{(t)} \} $ are the coefficients of the \textit{signal learning} and \textit{noise memorization}, respectively. To provide a more precise characterization of the coefficients, we define $\overline{\rho}_{j,r,i}^{(t)} \triangleq \rho_{j,r,i}^{(t)}\mathds{1} (\rho_{j,r,i}^{(t)} \geq 0)$, $\underline{\rho}_{j,r,i}^{(t)} \triangleq \rho_{j,r,i}^{(t)} \mathds{1} (\rho_{j,r,i}^{(t)} \leq 0)$. We refer to Equation (\ref{eq:w_decomposition}) as the \textit{signal-noise decomposition} of $\mathbf{w}_{j,r}^{(t)}$. The normalization factors $\| \boldsymbol{\mu} \|^{-2}_{2} $ and $\| \boldsymbol{\xi}_i \|^{-2}_{2} $ are introduced to ensure that $\gamma_{j,r}^{(t)} \approx \langle \mathbf{w}_{j,r}^{(t)}, \boldsymbol{\mu} \rangle $, and $\rho_{j,r,i}^{(t)} \approx \langle \mathbf{w}_{j,r}^{(t)}, \boldsymbol{\xi}_i \rangle $. We use $\gamma_{j,r}^{(t)} $ to characterize the process of \textit{signal learning} and $\rho_{j,r,i}^{(t)}$ to characterize the \textit{noise memorization}. If certain $ \gamma_{j,r}^{(t)} $ values are sufficiently large while all $|\rho_{j,r,i}^{(t)}|$ remain relatively small, this indicates that the network is primarily learning the label through signal learning. Conversely, if some $|\rho_{j,r,i}^{(t)}|$ values are large while all $ \gamma_{j,r}^{(t)} $ values remain small, the network is primarily focused on noise memorization.

\section{Theoretical Results} \label{sec:theoretical_results}

In this section, we introduce our key theoretical findings that demonstrate the optimization and generalization of GCNs through feature learning. Our analysis is based on the following assumptions:

\begin{assumption} \label{condition:d_sigma0_eta}
Suppose that 
    (1) The dimension $d$ is sufficiently large: $d = \tilde{\Omega}(m^{2\vee [4/(q-2)]}n^{4 \vee [(2q-2)/(q-2)]})$.
    (2) The size of the training sample $n$ and width of GCNs $m$ adhere to $n,m = \Omega(\mathrm{polylog}(d))$. 
    (3) The edge probability $p,s = \Omega ( {\sqrt{\log(n)/{n}})} $ and $ \Xi \triangleq \frac{p-s}{p+s} $ is a positive constant. (4) The learning rate $\eta$ satisfies $\eta \leq  \tilde{O}(\min\{\|\boldsymbol{\mu}\|_{2}^{-2}, \sigma_{p}^{-2}d^{-1}\})$, and the weight initialization strength follows $ \sigma_{0} \leq \tilde{O}(m^{-2/(q-2)}n^{-[1/(q-2)]\vee 1}  \cdot \min \{(\sigma_{p}\sqrt{d/(n(p+s))})^{-1}, \Xi^{-1}\|\boldsymbol{\mu} 
 \|_{2}^{-1}\}$.
\end{assumption}

The rationale for these assumptions is as follows: (1) The requirement for a high dimension is specifically aimed at ensuring that the learning occurs in a sufficiently over-parameterized setting when the second layer remains fixed. Similar choice can be found in \cite{cao2022benign,kou2023benign}. (2) This condition guarantees that certain statistical properties of the training data and weight initialization hold with high probability at least $1 - d^{-1}$. 
(3) The condition on edge probabilities $p$ and $s$ guarantees a sufficient concentration in node degrees and captures the level of homophily in the graph data. (4) The condition on $\eta$ and $\sigma_0$ ensures that gradient descent can effectively minimize the training loss.

Furthermore, we introduce a critical quantity called the signal-to-noise ratio (SNR), which measures the relative learning speed between the signal and the noise. It is defined as $ \mathrm{SNR} =  \| \boldsymbol{\mu} \|_2/(\sigma_p \sqrt{d})$. To prepare for our main result, we also define an effective SNR for GNNs as $ \mathrm{SNR}_{G} =  \| \boldsymbol{\mu} \|_2/(\sigma_p \sqrt{d}) \cdot (n(p+s))^{(q-2)/(2q)} $. Given the above assumptions and definitions, we present our main result for GNNs as follows:

\begin{theorem}\label{thm:signal_learning_main}
Let $T = \tilde{\Theta}( \eta^{-1} m\sigma_0 ^{-(q-2)} \Xi^{-q} \| \boldsymbol{\mu} \|_2^{-q} + \eta^{-1}\epsilon^{-1} m^{3}\|\boldsymbol{\mu}  \|_2^{-2})$. Under Assumption~\ref{condition:d_sigma0_eta}, if $ n \cdot \mathrm{SNR}_G^q  = \tilde\Omega( 1 )$, then with probability at least $ 1 - d^{-1}$, there exists a time $0 \leq t \leq T$ such that:
\begin{itemize} [leftmargin = *]
    \item The GCN learns the signal: $\max_r\gamma_{j,r}^{(t)} =  {\Omega}(1)$ for $j\in \{\pm 1\}$.
    \item The GCN does not memorize the noises: 
    $\max_{j,r,i} |\rho_{j,r,i}^{(T)}| = \tilde O( \sigma_0 \sigma_p \sqrt{d/n(p+s)} )$.
    \item The training loss converges to $\epsilon$, i.e., $L^{\mathrm{GCN}}_\mathcal{S}(\mathbf{W}^{(t)}) \leq \epsilon$.
    \item The trained GCN achieves a small test loss: $L^{\mathrm{GCN}}_{\mathcal{D}}(\mathbf{W}^{(t)})\leq  {c_1\epsilon + \exp(-c_2 n^{2})}$, where $c_1$ and $c_2$ are positive constants.
\end{itemize}

\end{theorem}

\begin{figure}[hbt]
\centering
\includegraphics[width=6.0cm]{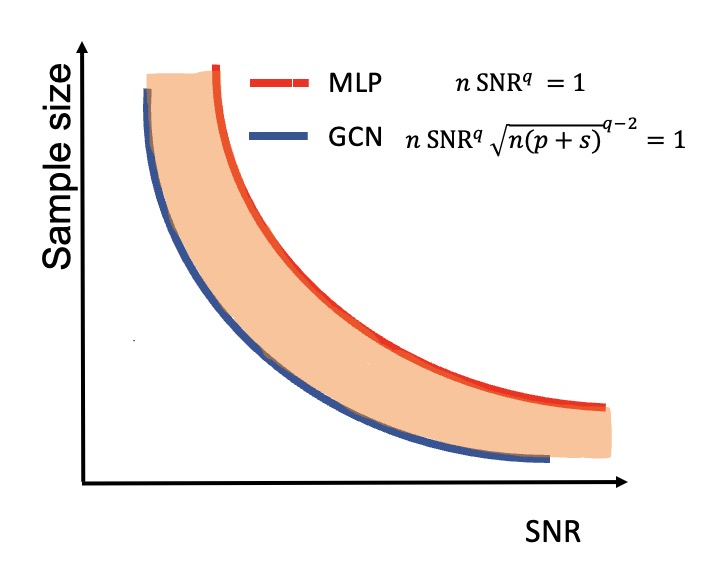}
\caption{Illustration of test performance comparison between GNN and MLP. The region below the red curve but above the blue curve (highlighted in orange) indicates the area where GNN can generalize well, but MLP fails to generalize to the test set.}
\label{fig:phase}
\end{figure}

Theorem~\ref{thm:signal_learning_main} reveals that, provided $ n \cdot \mathrm{SNR}_G^q = \tilde\Omega( 1 )$, the GCN is capable of learning the signal, achieving $\max_r\gamma_{j,r}^{(t)} = \Omega(1)$. On the other hand, the noise memorization during gradient descent training is suppressed, as indicated by $\max_{j,r,i} |\rho_{j,r,i}^{(T)}| = \tilde O( \sigma_0 \sigma_p \sqrt{d/n(p+s)} )$, given that $\sigma_0 \sigma_p \sqrt{d/n(p+s)} \ll 1$, according to assumption \ref{condition:d_sigma0_eta}. Because the signal learned by the network is sufficiently strong and much larger than the noise memorization, the model can generalize well to the test samples. Consequently, the learned neural network achieves both low training and test losses.  

We compare the result in Theorem \ref{thm:signal_learning_main} with the result of MLPs presented in \cite{cao2022benign} to highlight the quantitative advantage in generalization of GCNs over MLPs. According to \cite{cao2022benign}, when $n \cdot \mathrm{SNR}^q = \tilde{\Omega}(1)$, MLPs can achieve a small test error; otherwise, they experience a large test error when $n \cdot \mathrm{SNR}^q = \tilde{O}(1)$. Together, these results highlight the differences in generalization, showing that GCNs have a broader regime for achieving low test errors, which is qualitatively represented by $ [n(p+s)]^{(q-2)/(2q)}$. This outcome further requires a dense setting, where $n(p+s) > 1$, which is consistent with Assumption \ref{condition:d_sigma0_eta}. The reason for this improved performance is that, during feature learning, graph convolution can effectively suppress noise memorization, enabling GCNs to generalize better than MLPs, especially in low SNR settings. 
Finally, we illustrate the quantitative difference established in this work through Figure~\ref{fig:phase}. By precisely characterizing the feature learning dynamics from optimization to generalization for GNNs, we have successfully shown how GCNs gain a significant advantage over MLPs due to the benefits provided by graph convolution.



\section{Proof Roadmap  }\label{section:tech_overview}


In this section, we present proof sketches for GCNs using feature learning theory. We discuss the primary challenges encountered in the study of GNNs and outline the key techniques used in our proofs to address these challenges. These techniques are further elaborated in the subsequent sections, and detailed proofs can be found in the appendix. Although we adopt the feature learning framework from prior works \cite{allen2020towards,cao2022benign}, our focus is fundamentally different, making the existing results not directly applicable to GCNs.


\subsection{Iterative of coefficients }

To analyze the feature learning process in GCNs, we introduce an iterative methodology based on the signal-noise decomposition (\ref{eq:w_decomposition}) and gradient descent update rule (\ref{eq:gcn_gdupdate}). The following lemma provides a means to track the iteration of signal learning and noise memorization under graph convolution:
\begin{lemma}\label{lemma:coefficient_iterative}
The coefficients $\gamma_{j,r}^{(t)}, \overline{\rho}_{j,r,i}^{(t)},\underline{\rho}_{j,r,i}^{(t)}$ with the initialization value $\gamma_{j,r}^{(0)},\overline{\rho}_{j,r,i}^{(0)},\underline{\rho}_{j,r,i}^{(0)} = 0$ in decomposition~(\ref{eq:w_decomposition}) adhere to the following equations:
\begin{align}
    &\gamma_{j,r}^{(t+1)} = \gamma_{j,r}^{(t)} - \frac{\eta}{nm}   \sum_{i=1}^n   {\ell}_i'^{(t)}   \sigma'(\langle \mathbf{w}_{j,r}^{(t)},  \Tilde{y}_i   {\boldsymbol{\mu}_i } \rangle)    y_i \Tilde{y}_i   \| \boldsymbol{\mu} \|_2^2,\label{eq:update_gamma1}  \\
    & \overline{\rho}_{j,r,i}^{(t+1)}   = \overline{\rho}_{j,r,i}^{(t)} - \frac{\eta}{nm}   \sum_{k \in \mathcal{N}(i)} D^{-1}_{k}    {\ell}_k'^{(t)}  \sigma' (  \langle \mathbf{w}_{j,r}^{(t)},  \tilde{\boldsymbol{\xi}}_{k}   \rangle  )   \| \boldsymbol{\xi}_i \|_2^2 \nonumber \\
     & \qquad \mathds{1} (y_{k} = j), \label{eq:update_zeta1} \\
   & \underline{\rho}_{j,r,i}^{(t+1)}   = \underline{\rho}_{j,r,i}^{(t)} + \frac{\eta}{nm}  \sum_{k \in \mathcal{N}(i)} D^{-1}_{k}     {\ell}_k'^{(t)}   \sigma' (  \langle \mathbf{w}_{j,r}^{(t)},  \tilde{\boldsymbol{\xi}}_{k}   \rangle  )     \| \boldsymbol{\xi}_i \|_2^2 \nonumber  \\
   & \qquad \mathds{1} (y_{k} = -j). \label{eq:update_omega1}
\end{align} 
\end{lemma}

Lemma~\ref{lemma:coefficient_iterative} simplifies the analysis of feature learning in GCNs by reducing it to the examination of the discrete dynamical system defined by Equations (\ref{eq:update_gamma1} - \ref{eq:update_omega1}). 
Our proof strategy emphasizes an in-depth evaluation of the coefficient values $\gamma_{j,r}^{(t)},\overline{\rho}_{j,r,i}^{(t)}, \underline{\rho}_{j,r,i}^{(t)}$ throughout the training process.

\subsection{The importance of dense graphs in trajectory analysis}\label{subsection:twostage}

Note that graph convolution aggregates information from neighboring nodes to the central node, which often leads to a loss of \textit{statistical concentration} for the aggregated noise vectors and labels. To mitigate this challenge, we utilize a dense graph structure by setting the edge probability $p,s =\Omega(\sqrt{\log(n)/n})$, as stated in Assumption \ref{condition:d_sigma0_eta}. As a result, we show that the node degrees exhibit good concentration properties, as demonstrated by the following lemma:

\begin{lemma}  \label{lem:degree_main}
Let $p, s = \Omega \left( \sqrt{\frac{\log (n/\delta)}{n}} \right)$ and $\delta > 0$, then with probability at least $1 - \delta$, we have
$ 
  n(p+s)/4 \le  D_{i} \le  3n(p+s)/4,  
$ 
where $D_i$ is the degree of node $i$.
\end{lemma}
Lemma \ref{lem:degree_main} establishes the concentration of node degrees, which is crucial for the trajectory analysis of iterations for both signal learning and noise memorization.

\subsection{The Role of Homophily}

To preserve the sign of the graph-aggregated labels, we introduce ``homophily'' by setting 
$p > s$. This setting ensures that the convolution process effectively integrates neighborhood label information into the central node. The formal result is provided in the following lemma:
\begin{lemma}\label{lemma:graph_numberofdata_main}
Suppose that $\delta > 0$, $p > s$, and $n  \geq 8 \frac{p+s}{(p-s)^2} \log(4/\delta)$. Then with probability at least $1 - \delta$, it holds that $
  \frac{1}{2} \frac{p-s}{p+s} |y_i| \le   |\Tilde{y}_i| \le \frac{3}{2} \frac{p-s}{p+s} |y_i|$.
\end{lemma}

Lemma \ref{lemma:graph_numberofdata_main} shows that the effective label after graph convolution retains the same sign as the original node label. It is worth noting that if we consider heterophily (where $s > p$), the same generalization results hold in our setting, except that the sign of the aggregated label will be opposite to that of the central node.

\subsection{Optimization analysis}

\begin{figure*}[htbp] 
\centering
\begin{minipage}{0.23\textwidth}
\includegraphics[width =1.3in]{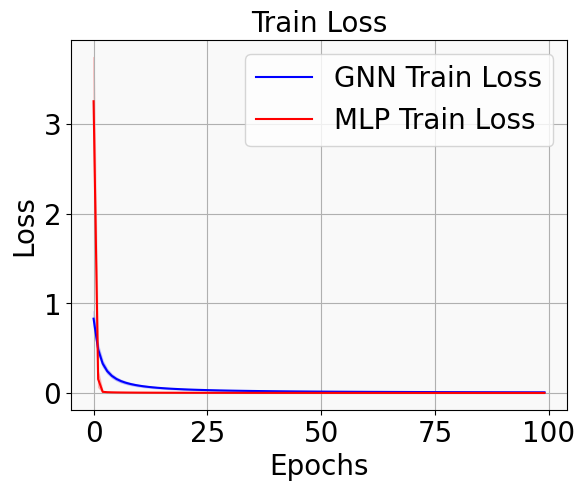}
\end{minipage} 
\hspace{2pt}
\begin{minipage}{0.23\textwidth}
\includegraphics[width =1.3in]{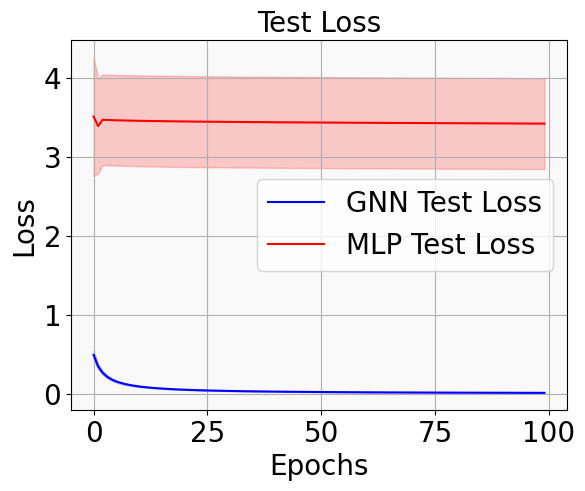}
\end{minipage}
\hspace{2pt}
\begin{minipage}{0.24\textwidth}
\includegraphics[width =1.38in]{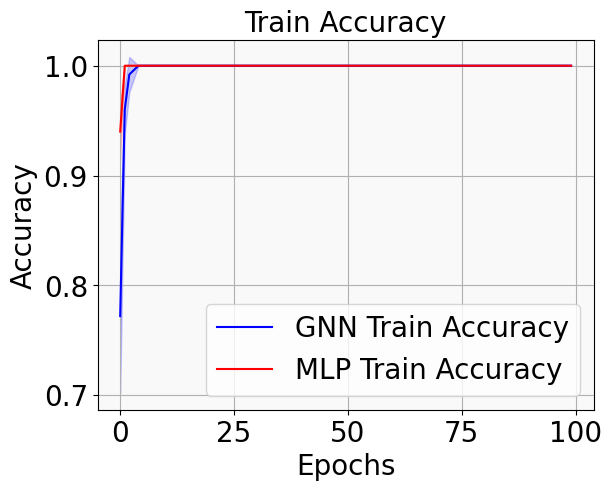}
\end{minipage}
\begin{minipage}{0.24\textwidth}
\includegraphics[width =1.39in]{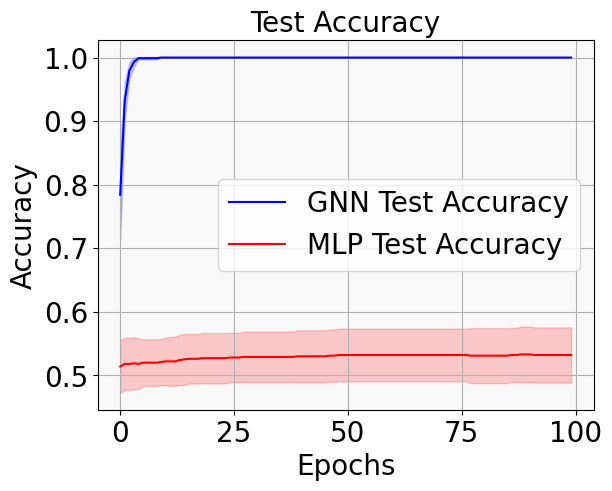}
\end{minipage}
\caption{Training loss, test loss, training accuracy, and test accuracy for both MLP and GNN over a span of 100 training epochs. Each curve represents the average of five experimental runs, with shaded regions indicating the error bars to capture variability.}
\label{fig:overfitting}
\end{figure*}

We provide a two-stage dynamics analysis to track the trajectory of the coefficients for signal learning and noise memorization, using Lemma \ref{lemma:coefficient_iterative}, Lemma \ref{lem:degree_main}, and Lemma \ref{lemma:graph_numberofdata_main}.

\paragraph{Stage 1.} Intuitively, the initial neural network weights are small enough such that the network at initialization exhibits constant-level loss derivatives on all training data points $\ell_i'^{(0)} = \ell'[ y_i \cdot f(\mathbf{W}^{(0)}, \mathbf{A},  {\mathbf{x}}_i)] = \Theta(1)$ for all $i\in [n]$. This is guaranteed under Assumption~\ref{condition:d_sigma0_eta} on $\sigma_0$. Motivated by this observation, the dynamics of the coefficients in Equations (\ref{eq:update_gamma1} - \ref{eq:update_omega1}) can be greatly simplified by replacing the $\ell_i'^{(t)}$ factors by their constant upper and lower bounds. The following lemma summarizes our main conclusion for feature learning in Stage 1:

\begin{lemma}\label{lemma:phase1_main_sketch}
Under the same conditions as Theorem~\ref{thm:signal_learning_main}, there exists $T_1 = \tilde O(\eta^{-1}m\sigma_{0}^{2-q}\Xi^{-q} \|\boldsymbol{\mu}\|_{2}^{-q}) $ such that $\max_{ r}\gamma_{j, r}^{(T_{1})} = \Omega(1)$ for $j\in \{\pm 1\}$, and $|\rho_{j,r,i}^{(t)}| = O\left(\sigma_0 \sigma_p \sqrt{d}/\sqrt{n(p+s)} \right)$ for all $j\in \{\pm 1\}$, $r\in[m]$, $i \in [n]$ and $0\leq t\leq T_1$. 
 
\end{lemma}

The proof can be found in Appendix \ref{sec:stage_1}.  Lemmas~\ref{lemma:phase1_main_sketch} leverages the period of training when the derivatives of the loss function remain at a constant order. 
It is important to note that graph convolution plays a significant role in differentiating the learning speeds between signal learning and noise memorization. Originally, without graph convolution, the learning speeds are primarily determined by $ \| \boldsymbol{\mu} \|_2$ and $\| \boldsymbol{\xi} \|_2$ for the signal and noise, respectively \cite{cao2022benign}. In contrast, with graph convolution, the learning speeds are determined by $ |\tilde{y}|  \| \boldsymbol{\mu} \|_2$ and $\| \tilde{\boldsymbol{\xi}} \|_2$ respectively. Here, $ |\tilde{y}|  \| \boldsymbol{\mu} \|_2$ is close to $ \| \boldsymbol{\mu} \|_2$, but $\| \tilde{\boldsymbol{\xi}} \|_2$ is much smaller than $\| \boldsymbol{\xi} \|_2$ (see Figure \ref{fig:agg} for an illustration). This implies that graph convolution can slow down noise memorization, thereby allowing GNNs to focus more on signal learning.

\paragraph{Stage 2.} Building on the results from the first stage, we move to the second stage of the training process. In this stage, the loss derivatives are no longer constant, and we show that the training error can be minimized to an arbitrarily small value. Besides, the scale differences established during the first stage of learning are preserved throughout the second stage:

\begin{lemma}\label{lemma:signal_proof_sketch}
Under the same conditions as Theorem~\ref{thm:signal_learning_main}, for any $t\in [T_1,  T]$, it holds that  $\max_{r}\gamma_{j, r}^{(T_{1})} \geq 2, \forall j \in \{\pm 1\}$ and   $|\rho_{j,r,i}^{(t)}| \leq \sigma_{0}\sigma_{p}\sqrt{d/(n(p+s))}$
for all $j\in\{\pm 1\}$, $r\in [m]$ and $i\in[n]$. Moreover, there exists a $t\in [T_1,  T]$ such that
  $ L^{\mathrm{GCN}}_{\mathcal{S}}(\mathbf{W}^{(t)})   \leq   \epsilon$.
 
\end{lemma}

Lemma~\ref{lemma:signal_proof_sketch} presents two primary outcomes: (1) Throughout this training phase, it ensures that the noise memorization coefficients remain significantly small, while the coefficients of the signal learning reach large values. (2) It guarantees convergence for the GNN, showing that the training loss can be reduced to an arbitrarily small value.

\subsection{Test error analysis}

Analyzing the generalization performance of graph neural networks is a challenging task. To tackle this issue, we introduce an expectation over the distribution for a single data point. Specifically, we consider a new data point $(\mathbf{x},y)$ drawn from the SNM-SBM distribution. The following lemma provides an upper bound on the test loss of GNNs after training:
\begin{lemma}\label{lemma:signal_polulation_loss_main}
Let $T$ be defined as in Theorem~\ref{thm:signal_learning_main}. Under the same conditions as Theorem~\ref{thm:signal_learning_main}, for any $t \leq T$ with $L^{\mathrm{GCN}}_{\mathcal{S}}(\mathbf{W}^{(t)}) \leq 1$, it holds that $L^{\mathrm{GCN}}_{\mathcal{D}}(\mathbf{W}^{(t)}) \leq c_1 \cdot L^{\mathrm{GCN}}_{\mathcal{S}}(\mathbf{W}^{(t)}) + \exp(-c_2 n^{2})$, where $c_1$ and $c_2$ are positive constants.
\end{lemma}

Lemma \ref{lemma:signal_polulation_loss_main} demonstrates that GNNs can achieve a small test error. Combined with the results stated in Lemma~\ref{lemma:signal_proof_sketch}, this completes the proof for Theorem~\ref{thm:signal_learning_main}.

\begin{figure*}[htbp]
\centering
    \subfigure[Performance of MLP]{\includegraphics[width=0.4 \textwidth]{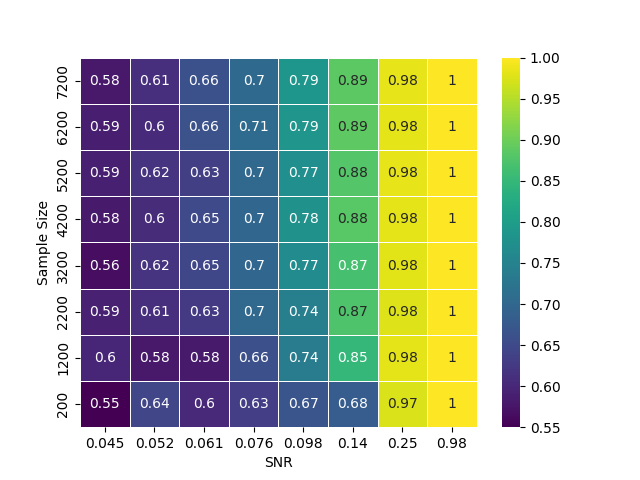}}
\subfigure[Performance of GCN]{\includegraphics[width=0.4 \textwidth]{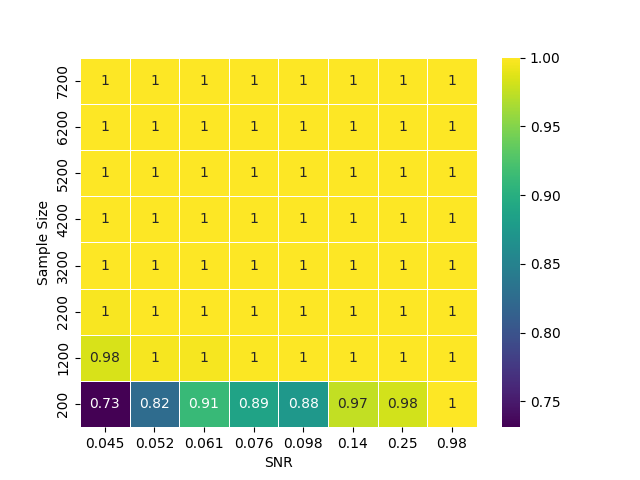}}
\caption{Test accuracy heatmap for MLPs and GCNs after training.}
\label{fig:heatmap}
\end{figure*}

\begin{figure*}[h] 
\centering
\begin{minipage}{0.23\textwidth}
\includegraphics[width =1.3in]{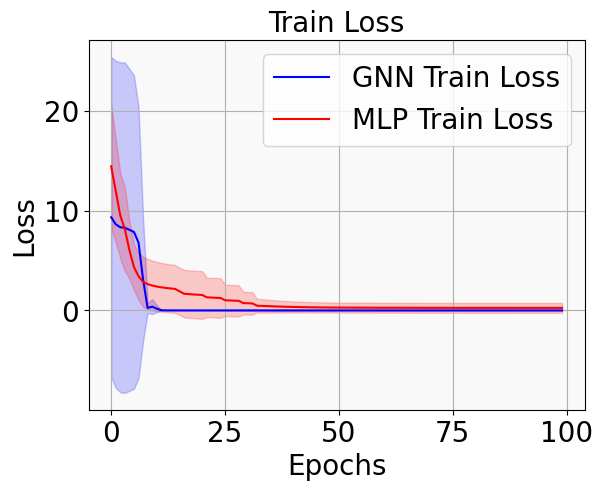}
\end{minipage} 
\hspace{2pt}
\begin{minipage}{0.23\textwidth}
\includegraphics[width =1.3in]{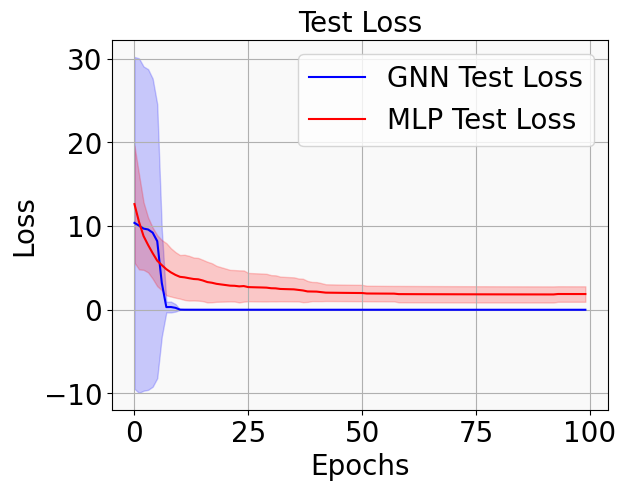}
\end{minipage}
\hspace{2pt}
\begin{minipage}{0.24\textwidth}
\includegraphics[width =1.38in]{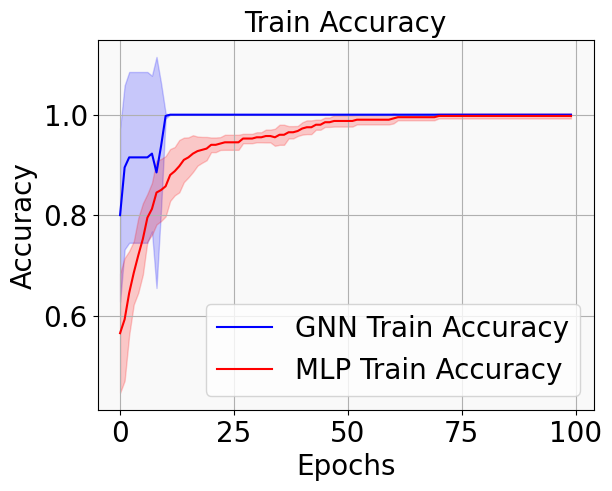}
\end{minipage}
\begin{minipage}{0.24\textwidth}
\includegraphics[width =1.39in]{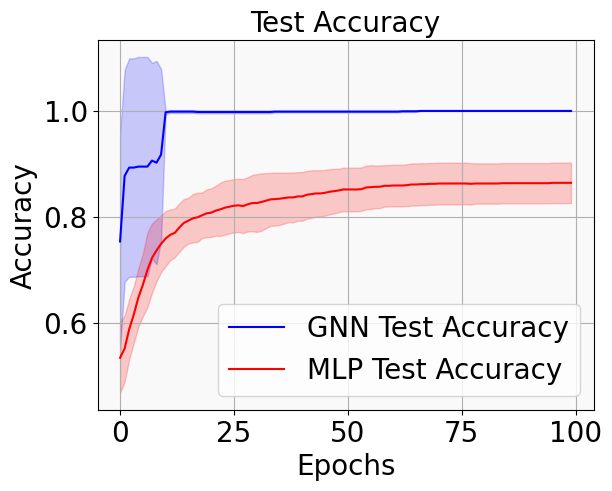}
\end{minipage}
\caption{The verification of our theoretical result with a modified real-world data. We show the training loss, testing loss, training accuracy, and testing accuracy for both MLP and GNN over a span of 100 training epochs. {Five experimental runs are conducted, with shaded areas highlighting error bars for variability.}}
\label{fig:mnist_graph}
\end{figure*}

\begin{figure}[hbtp]
\centering
\includegraphics[width=8.0cm]{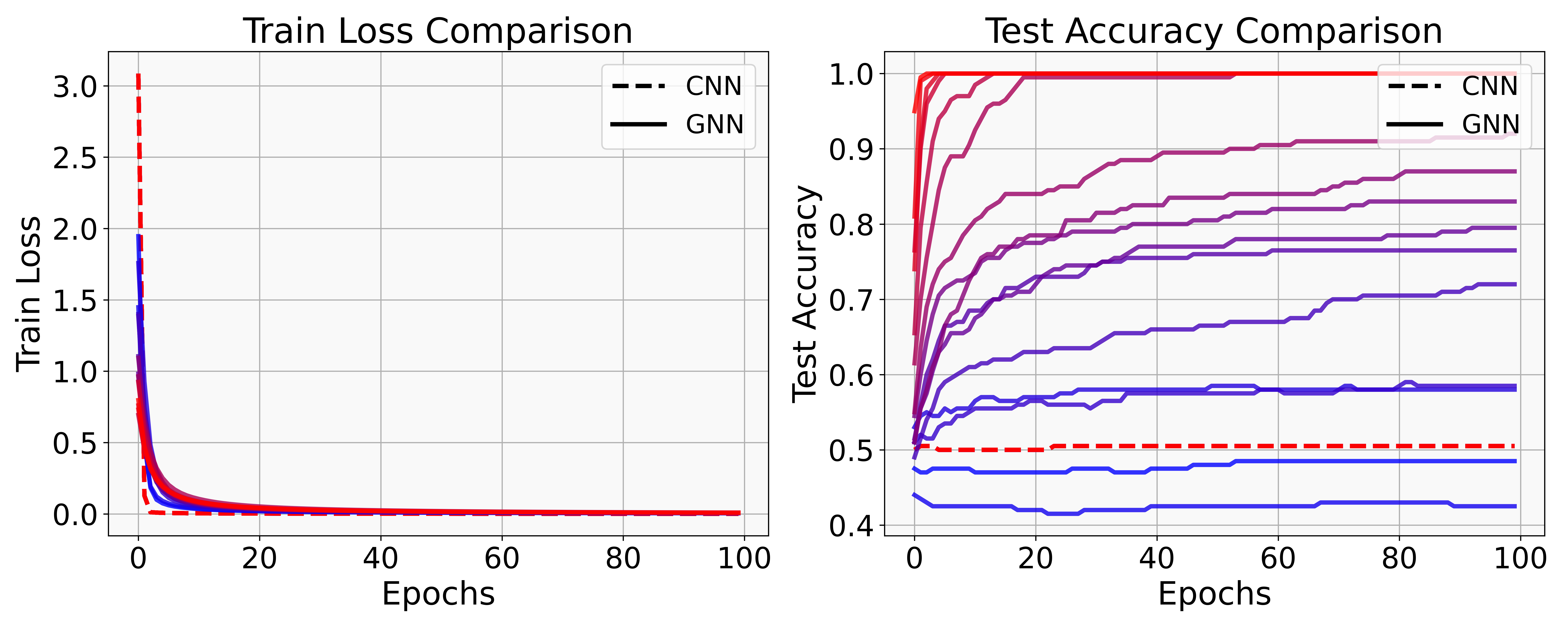}
\caption{Training loss and test accuracy comparison between CNN (dashed lines) and GNN (solid lines) models across varying graph densities.} 
\label{fig:dense}
\end{figure}

\begin{figure}[hbtp]
\centering
\includegraphics[width=8.0cm]{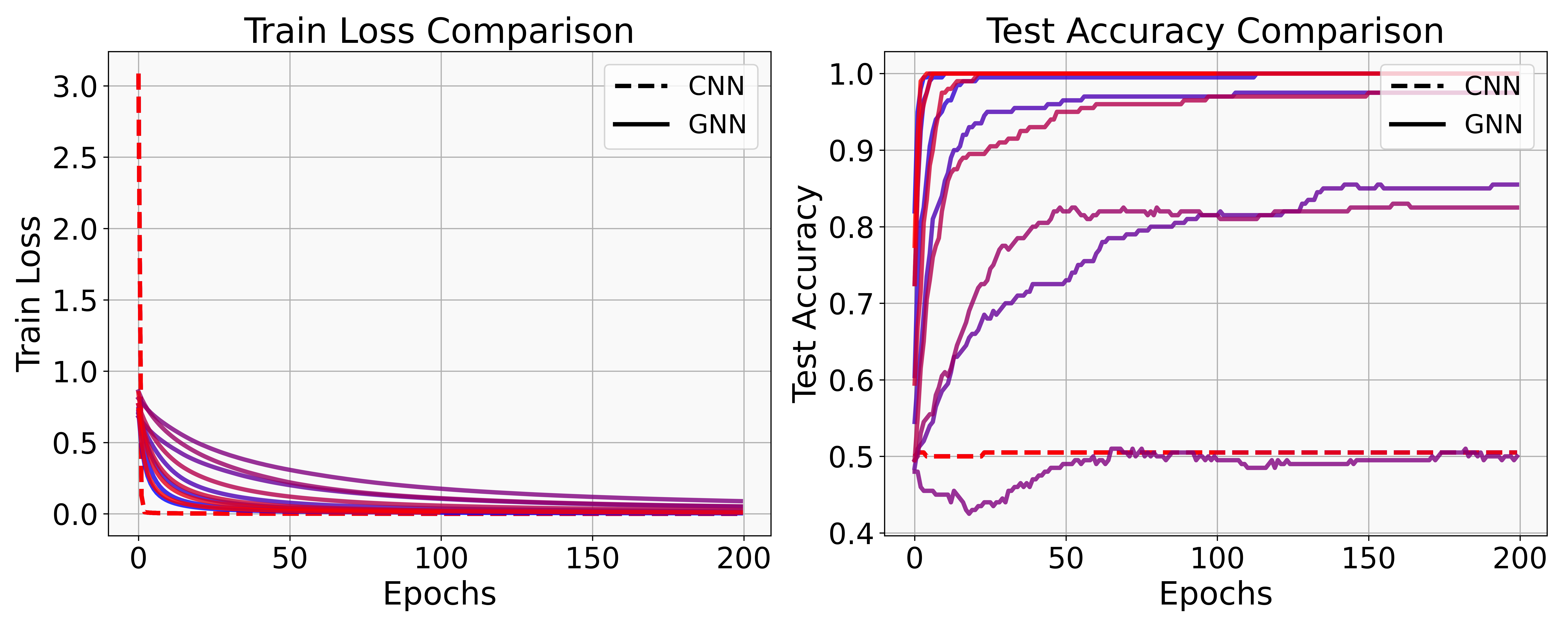}
\caption{Training loss and test accuracy comparison between CNN (dashed lines) and GNN (solid lines) models under varying homophily levels while keeping graph density constant.}
\label{fig:homo}
\end{figure}

\section{Experiments} \label{sec:experiments}

In this section, we validate our theoretical findings through numerical simulations using synthetic data and modified real-world data.

\paragraph{Synthetic data.} We generated synthetic data using the SNM-SBM model. The signal vector $\boldsymbol{\mu}$ is drawn from a standard normal distribution $\mathcal{N}(\mathbf{0}, \mathbf{I})$, and the noise vector $\boldsymbol{\xi}$ is sampled from a Gaussian distribution $\mathcal{N}(\mathbf{0}, 20\mathbf{I})$. We set the training data size to $n = 50$, the input dimension to $d = 500$, the edge probability to $p = 0.5$, and $s = 0.08$.
We train a two-layer MLP and GNN using Equation (\ref{eq:gcn}) with polynomial ReLU activation $q = 3$. The optimization is by the gradient descent method with a learning rate of $\eta = 0.03$. The primary task is node classification, aiming to predict the class labels of nodes in a graph. Figure \ref{fig:overfitting} shows the training loss, test loss, training accuracy, and test accuracy for both the MLP and GNN. Our observations reveal that both the GNN and MLP can achieve zero training error. However, while the GNN attains nearly zero test error, the MLP fails to generalize. This simulation result validates our theoretical findings.

\paragraph{Graph Density}

 We maintained a fixed ratio of \( p/s = 10 \), varying only the absolute values of \( p \) and \( s \). Specifically, we considered values of \( p \) in the range \( \{0.01,\, 0.02,\, 0.03,\, \dots,\, 0.1,\, 0.2,\, \dots,\, 0.7\} \). This experimental setup preserves the relative homophily while systematically adjusting the overall graph density. The results, depicted in Figure~\ref{fig:dense}, utilize a color gradient from blue (sparser graphs) to red (denser graphs) to highlight changes in model performance. We observe a significant improvement in test accuracy as the graph becomes denser. Beyond a certain density threshold, the model consistently achieves near-perfect test accuracy. These findings empirically support our claim that dense graphs facilitate effective aggregation.

\paragraph{Graph Homophily}
We fixed the sum \( p + s \) to a constant value, varying only the individual values of \( p \) and \( s \). Specifically, we considered values of \( p \) within the range \(\{0.1,\, 0.15,\, 0.2,\, \dots,\, 0.6\}\). This configuration maintains a constant overall graph density while altering the level of homophily. Results are presented in Figure~\ref{fig:homo}, where we employ color coding from blue (strongly homophilic) to red (strongly heterophilic). We observe that both strongly homophilic and strongly heterophilic graph structures achieve higher test accuracy, whereas graphs with intermediate homophily exhibit relatively lower performance. These findings empirically support our theoretical claims concerning the effects of homophily.

\paragraph{Heatmap of test accuracy} 

We then explore a range of Signal-to-Noise Ratios (SNRs) from 0.045 to 0.98, and a variety of sample sizes, $n$, ranging from 200 to 7200. Based on our results, we train the neural network for 200 steps for each combination of SNR and sample size $n$. After training, we calculate the test accuracy for each run. The results are presented as a heatmap in Figure \ref{fig:heatmap}. Compared to MLPs, GCNs demonstrate a perfect accuracy score of 1 across a more extensive range in the SNR and $n$ plane, indicating that GNNs have a broader \textit{high test accuracy} regime with high test accuracy. 

{\paragraph{Real-world data.} We conduct an experiment using the MNIST dataset. To align with the theoretical setting, we added Gaussian noise sampled from $\sigma^2_p \mathbf{I}$ to the digit images and then divide both the noise component and the digit component into two groups of patches of equal size, inspired by \cite{cao2022benign}. The details for creating this data can be found in Appendix \ref{seca:mnist}. We select the digits `1' and `2' from the ten MNIST digits, using a training sample size of $n=100$, while the remaining samples are used as the test set. The graph structure was generated using a stochastic block model, with an edge probability of $p = 0.2$, and $s = 0.01$. We set the learning rate $\eta = 0.0005$ and the noise level $\sigma_p = 0.1$. Detailed results are shown in Figure \ref{fig:mnist_graph}. The results are consistent with our theoretical conclusions, reinforcing the insights derived from our analysis.

\section{Conclusion and Limitation} \label{sec:conclusion}

This paper leverages feature learning theory to analyze the optimization and generalization behavior of GCNs. Specifically, we adopt a signal-noise decomposition to characterize the signal learning and noise memorization processes during the training of a two-layer GCN. We establish specific conditions under which a GNN primarily focuses on signal learning, resulting in low training and testing errors. When combined with results for MLPs, our findings quantitatively demonstrate that GCNs, by utilizing structural information, outperform MLPs in terms of generalization ability across a broader benign regime.

\paragraph{Limitation} Our theoretical framework is limited to analyzing the role of graph convolution within a specific two-layer GCN and a particular data model. In practice, the feature learning dynamics of neural networks can be influenced by various factors, such as the depth of the GNN, the choice of activation function, the optimization algorithm, and the underlying data distribution \cite{kou2023benign,zou2021understanding,zou2023benefits}. Future work could extend our framework to account for the impact of these additional factors on feature learning in GCNs.

\section*{Acknowledgments}
  We thank the anonymous reviewers for their insightful comments to improve the paper. WH was supported by JSPS KAKENHI Grant Number 24K20848. YC was supported in part by NSFC 12301657 and Hong Kong ECS award 27308624. TS was partially supported by JSPS KAKENHI (24K02905) and JST CREST (JPMJCR2015).

\bibliography{reference}
\bibliographystyle{plain}

\section*{Checklist}



 \begin{enumerate}

 \item For all models and algorithms presented, check if you include:
 \begin{enumerate}
   \item A clear description of the mathematical setting, assumptions, algorithm, and/or model. [Yes] Please refer to the corresponding context in Sections \ref{sec:problem_setup} and \ref{sec:theoretical_results}.
   \item An analysis of the properties and complexity (time, space, sample size) of any algorithm. [Yes] We have shown the sample complexity in Theorem \ref{thm:signal_learning_main}.
   \item (Optional) Anonymized source code, with specification of all dependencies, including external libraries. [Yes] We have uploaded the code as supplementary material.
 \end{enumerate}

 \item For any theoretical claim, check if you include:
 \begin{enumerate}
   \item Statements of the full set of assumptions of all theoretical results. [Yes] We have stated the complete set of assumptions in Assumption \ref{condition:d_sigma0_eta}.
   \item Complete proofs of all theoretical results. [Yes] The proof sketch and complete proof are provided in Section \ref{section:tech_overview} and Appendices \ref{seca:preliminary}, \ref{seca:general}, \ref{seca:dynamics}, and \ref{seca:generalization}, respectively.
   \item Clear explanations of any assumptions. [Yes] We have provided explanations below Assumption \ref{condition:d_sigma0_eta}.   
 \end{enumerate}

 \item For all figures and tables that present empirical results, check if you include:
 \begin{enumerate}
   \item The code, data, and instructions needed to reproduce the main experimental results (either in the supplemental material or as a URL). [Yes] We have uploaded the code as supplementary material and provided instructions in Section \ref{sec:experiments} and Appendices \ref{seca:phase} and \ref{seca:mnist}.
   \item All the training details (e.g., data splits, hyperparameters, how they were chosen). [Yes] We have uploaded the code as supplementary material and provided details in Section \ref{sec:experiments} and Appendices \ref{seca:phase} and \ref{seca:mnist}.
         \item A clear definition of the specific measure or statistics and error bars (e.g., with respect to the random seed after running experiments multiple times). [Yes] We have included the error bars in Figures \ref{fig:overfitting} and \ref{fig:mnist_graph}.
         \item A description of the computing infrastructure used. (e.g., type of GPUs, internal cluster, or cloud provider). [Yes] We have included the computing infrastructure details in Appendix \ref{sec:comp}.
 \end{enumerate}

 \item If you are using existing assets (e.g., code, data, models) or curating/releasing new assets, check if you include:
 \begin{enumerate}
   \item Citations of the creator If your work uses existing assets. [Yes] We have provided the citation.
   \item The license information of the assets, if applicable. [Yes] We have uploaded the code to the supplementary material.
   \item New assets either in the supplemental material or as a URL, if applicable. [Not Applicable] This paper does not release new assets.
   \item Information about consent from data providers/curators. [Not Applicable] Our study does not involve external data requiring consent.
   \item Discussion of sensible content if applicable, e.g., personally identifiable information or offensive content. [Not Applicable] Our study does not include any sensitive or personally identifiable information.
 \end{enumerate}

 \item If you used crowdsourcing or conducted research with human subjects, check if you include:
 \begin{enumerate}
   \item The full text of instructions given to participants and screenshots. [Not Applicable] Our study does not involve any human participants.
   \item Descriptions of potential participant risks, with links to Institutional Review Board (IRB) approvals if applicable. [Not Applicable] There are no human subjects involved in our research.
   \item The estimated hourly wage paid to participants and the total amount spent on participant compensation. [Not Applicable] No participants were recruited or compensated.
 \end{enumerate}

 \end{enumerate}

\newpage
\appendix
\onecolumn

\section{Preliminary Lemmas} \label{seca:preliminary}

In this section, we present preliminary lemmas which form the foundation for the proofs to be detailed in the subsequent sections. The proof will be developed after the lemmas presented.

\subsection{Preliminary Lemmas without Graph Convolution}

{In this section, we introduce necessary lemmas that will be used in the analysis without graph convolution, following the study of feature learning in CNN \cite{cao2022benign}. In particular, Lemma \ref{lemma:data_innerproducts} states that noise vectors are ``almost orthogonal'' to each other and Lemma \ref{lemma:initialization_weight} indicates that random initialization results in a controllable inner product between the weights at initialization and the data vectors.}

\begin{lemma} \cite{cao2022benign} \label{lemma:data_innerproducts}
Suppose that $\delta > 0$ and $d = \Omega( \log(4n / \delta) ) $. Then with probability at least $1 - \delta$, 
\begin{align*}
    &\sigma_p^2 d / 2\leq \| \boldsymbol{\xi}_i \|_2^2 \leq 3\sigma_p^2 d / 2,\\
    & |\langle \boldsymbol{\xi}_i, \boldsymbol{\xi}_{i'} \rangle| \leq 2\sigma_p^2 \cdot \sqrt{d \log(4n^2 / \delta)},
\end{align*}
for all $i,i'\in [n]$.
\end{lemma}

\begin{lemma} \cite{cao2022benign} \label{lemma:initialization_weight} Suppose that $d  = \Omega( \log(nm / \delta) )$, $ m = \Omega(\log(1 / \delta))$. Then with probability at least $1 - \delta$, 
\begin{align*}
    &|\langle \mathbf{w}_{j,r}^{(0)}, \boldsymbol{\mu} \rangle | \leq \sqrt{2 \log(8m/\delta)} \cdot \sigma_0 \| \boldsymbol{\mu} \|_2,\\
    &| \langle \mathbf{w}_{j,r}^{(0)}, \boldsymbol{\xi}_i \rangle | \leq 2\sqrt{ \log(8mn/\delta)}\cdot \sigma_0 \sigma_p \sqrt{d}, 
\end{align*}
for all $r\in [m]$,  $j\in \{\pm 1\}$ and $i\in [n]$. Moreover, 
\begin{align*}
    \sigma_0 \| \boldsymbol{\mu} \|_2 / 2  & \leq \max_{r\in[m]} j\cdot \langle \mathbf{w}_{j,r}^{(0)}, \boldsymbol{\mu} \rangle \leq \sqrt{2 \log(8m/\delta)} \cdot \sigma_0 \| \boldsymbol{\mu} \|_2,\\
    \sigma_0 \sigma_p \sqrt{d} / 4 & \leq \max_{r\in[m]} j\cdot \langle \mathbf{w}_{j,r}^{(0)}, \boldsymbol{\xi}_i \rangle \leq 2\sqrt{ \log(8mn/\delta)} \cdot \sigma_0 \sigma_p \sqrt{d},
\end{align*}
for all $j\in \{\pm 1\}$ and $i\in [n]$.
\end{lemma}

\subsection{Preliminary Lemmas on Graph Properties}

{We now introduce important lemmas that are critical to our analysis. The key idea to ensure a relatively dense graph. In a sparser graph, the concentration properties of graph degree (Lemma \ref{lem:degree}), the graph convoluted label (\ref{lemma:graph_numberofdata}), the graph convoluted noise vector (Lemma \ref{lemma:initialization_norms} and Lemma \ref{lemma:graph_data_innerproducts}) are no longer guaranteed. This lack of concentration affects the behavior of coefficients during gradient descent training, leading to deviations from our current main results.}

\begin{lemma} [Degree concentration] \label{lem:degree}
Let $p, s = \Omega \left( \sqrt{\frac{\log (n/\delta)}{n}} \right)$ and $\delta > 0$, then with probability at least $1 - \delta$, we have
\begin{align*}
  n(p+s)/4 \le  D_{i} \le  3n(p+s)/4.  
\end{align*} 
\end{lemma}
\begin{proof}
It is known that the degrees are sums of Bernoulli random variables.
\begin{align*} 
    D_{i} = 1 + \sum_{j \neq i}^{n} a_{ij}, 
\end{align*}
where $a_{ij} = [\mathbf{A}]_{ij}$. Hence, by the Hoeffding's inequality, with probability at least $ 1 - \delta/n$
\begin{align*}
     | D_{i} - \mathbb{E}[D_{i}]| <  \sqrt{\log(n/\delta)(n-1) }.
\end{align*}
Note that $a_{ii} = 1$ is a fixed value, which means that it is not a random variable, thus the denominator in the exponential part is $n-1$ instead of $n$. Now we calculate the expectation of degree:
\begin{align*}
\mathbb{E}[D_{ii}] = 1+\frac{n}{2} s + (\frac{n}{2}-1)p  = n(p+s)/2 + 1-p,
\end{align*}
then we have
\begin{align*}
    \left| D_{i} -  n(p+s)/2 + 1 - p \right| \le  \sqrt{n \log (n/\delta)}. 
\end{align*}
Because that $p,s = \Omega \left( \sqrt{\frac{\log (n/\delta)}{n}} \right)$, we further have,
\begin{align*}
  n(p+s)/4 \le  D_{i} \le  3n(p+s)/4.  
\end{align*}
Applying a union bound over $i \in [n]$ conclude the proof.
\end{proof}

\begin{lemma}\label{lemma:graph_numberofdata}
Suppose that $\delta > 0$ and $n  \geq 8 \frac{p+s}{(p-s)^2} \log(4/\delta)$. Then with probability at least $1 - \delta$, 
\begin{align*}
  \frac{1}{2} \frac{p-s}{p+s} |y_i| \le   |\Tilde{y}_i| \le \frac{3}{2} \frac{p-s}{p+s} |y_i|.
\end{align*}
\end{lemma}
\begin{proof}[Proof of Lemma~\ref{lemma:graph_numberofdata}] By Hoeffding's inequality, with probability at least $1 - \delta / 2$, we have
\begin{align*}
    \Bigg| \frac{1}{D_i} \sum_{k \in \mathcal{N}(i)}  y_k  -\frac{p-s}{p+s} y_i \Bigg| \leq \sqrt{\frac{\log(4/\delta)}{2n(p+s)}}.
\end{align*}
Therefore, as long as $ n  \geq 8 \frac{p+s}{(p-s)^2} \log(4/\delta)$, we have:
\begin{align*}
   \frac{1}{2} \frac{p-s}{p+s} |y_i| \le   |\Tilde{y}_i| \le \frac{3}{2} \frac{p-s}{p+s} |y_i|.
\end{align*}
This proves the result for the stability of sign of graph convoluted label. 
\end{proof}

\begin{lemma}   \label{lemma:graph_data_innerproducts}
Suppose that $\delta > 0$ and $d  = \Omega(n^2(p+s)^2 \log(4n^2 / \delta) )$. Then with probability at least $1 - \delta$, 
\begin{align*}
    &\sigma_p^2 d /(4n(p+s)) \leq \| \Tilde{\boldsymbol{\xi}}_i \|_2^2 \leq 3\sigma_p^2 d /(4n(p+s)),
\end{align*}
for all $i \in [n]$.
\end{lemma}
\begin{proof}[Proof of Lemma~\ref{lemma:graph_data_innerproducts}] 
It is known that:
\begin{align*}
   \| \Tilde{\boldsymbol{\xi}}_i \|_2^2 = \frac{1}{D^2_i} \sum_{j=1}^d \left(\sum_{k =1}^{D_i} \xi_{jk} \right)^2 = \frac{1}{D^2_i} \sum_{j=1}^d  \sum_{k =1}^{D_i} \xi^2_{jk}  + \frac{1}{D^2_i} \sum_{j=1}^d  \sum_{k \neq k'}^{D_i} \xi_{jk'} \xi_{jk}.   
\end{align*}
By Bernstein's inequality, with probability at least $1 - \delta / (2n)$ we have
\begin{align*}
    \left|    \sum_{j=1}^d  \sum_{k =1}^{D_i} \xi^2_{jk}   - \sigma_p^2 d D_i   \right| = O(\sigma_p^2 \cdot \sqrt{d D_i \log(4n / \delta)}).
\end{align*}
Therefore, as long as $d  = \Omega( \log(4n / \delta)/(n(p+s))  )$, we have 
\begin{align*}
     3\sigma_p^2 d D_i /4  \leq  \sum_{j=1}^d  \sum_{k =1}^{D_i} \xi^2_{jk} \leq 5\sigma_p^2 d D_i /4.
\end{align*}
By Lemma \ref{lem:degree}, we have,
\begin{align*}
     2\sigma_p^2 d/(4n(p+s))  \leq \frac{1}{D^2_i}  \sum_{j=1}^d  \sum_{k =1}^{D_i} \xi^2_{jk} \leq  6\sigma_p^2 d/(4n(p+s)) .
\end{align*}

Moreover, clearly $\langle \boldsymbol{\xi}_k, \boldsymbol{\xi}_{k'} \rangle$ has mean zero. 
For any $k,k'$ with $k\neq k'$, by Bernstein's inequality, with probability at least $1 - \delta / (2n^2)$ we have
\begin{align*}
    | \langle \boldsymbol{\xi}_k, \boldsymbol{\xi}_{k'} \rangle | \leq 2\sigma_p^2 \cdot \sqrt{d \log(4n^2 / \delta)}.
\end{align*}
Applying a union bound we have that with probability at least $1 - \delta$,
\begin{align*}
|\langle  \boldsymbol{\xi}_k,  \boldsymbol{\xi}_{k'} \rangle| \leq 2\sigma_p^2 \cdot \sqrt{d \log(4n^2 / \delta)}.
\end{align*}
Therefore, as long as $d  = \Omega(n^2(p+s)^2 \log(4n^2 / \delta) )$, we have 
\begin{align*}
    &\sigma_p^2 d /(4n(p+s)) \leq \| \Tilde{\boldsymbol{\xi}}_i \|_2^2 \leq 3\sigma_p^2 d /(4n(p+s)).
\end{align*}

\begin{remark}
  {We compare the noise vector both before and after applying graph convolution. By examining Lemma \ref{lemma:data_innerproducts} and Lemma \ref{lemma:graph_data_innerproducts},  we discover that the expectation of the $\ell_2$ norm of noise vector is reduced by a factor of  $\sqrt{n(p+s)/2}$. This factor represents the square root of the expected degree of the graph, indicating a significant change in the noise characteristics as a result of the graph convolution process. We provide a demonstrative visualization in Figure \ref{fig:agg}.}
\end{remark}
\end{proof}

\begin{figure*}[h]
\centering
\includegraphics[width=8.1cm]{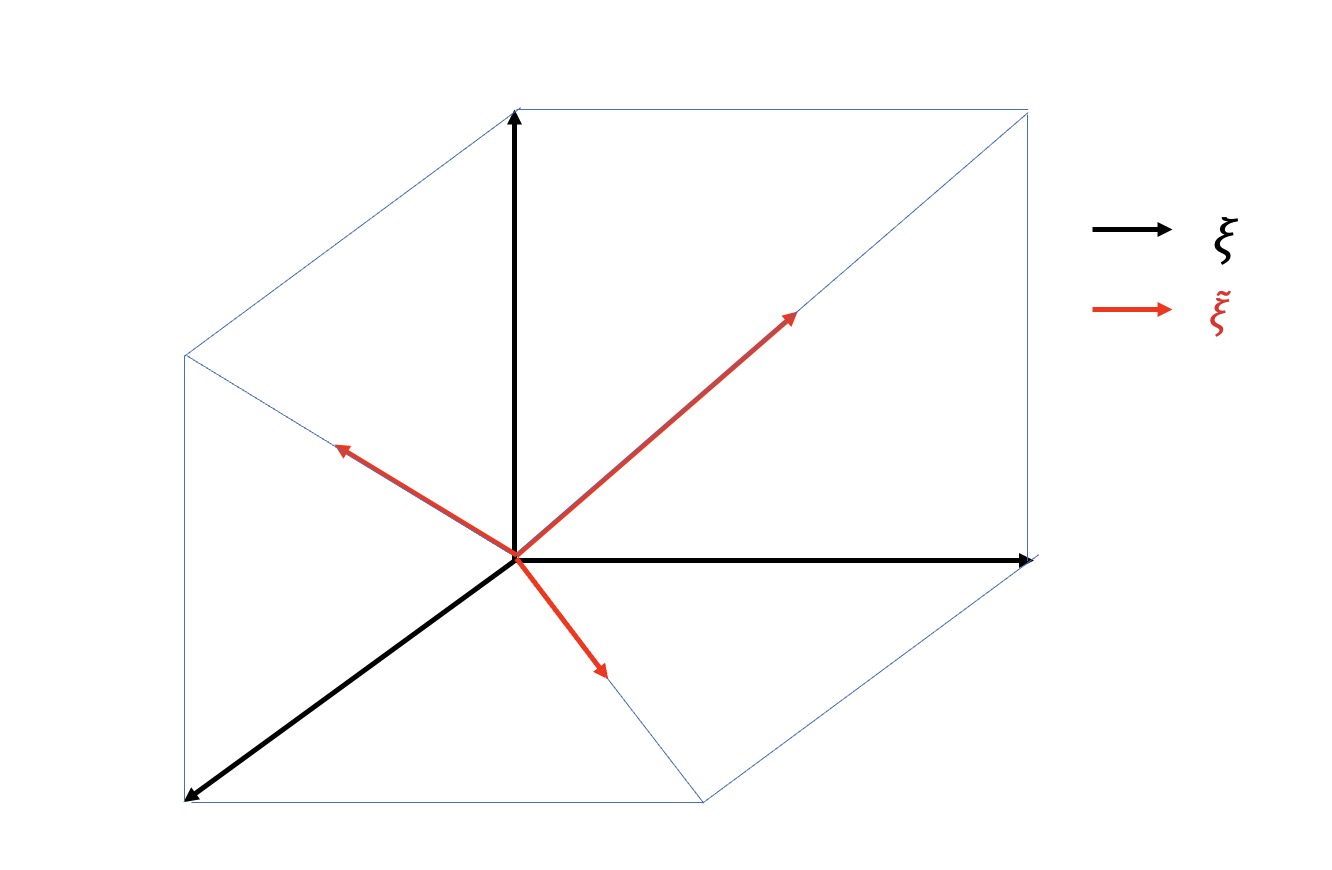}
\caption{{An illustrative example of noise vector before and after graph aggregation. In this example, we consider $d =3$ and all degree are 1. The black vectors stand for noise vectors $\boldsymbol{\xi}$ before graph convolution. Each of them are orthogonal to each other. The red vectors represent noise vectors after graph convolution $\tilde{\boldsymbol{\xi}}$. They are graph convoluted noise vectors of two original noise vectors. Note that the $\ell_2$ norm between two kinds of vector follows $  \|\Tilde{\boldsymbol{\xi}} \|_2 =\frac{\sqrt{2}}{2} \|\boldsymbol{\xi} \|_2 $. This plot demonstrates how graph convolution shrinks the $\ell_2$ norm of noise vectors.}}
\label{fig:agg}
\end{figure*}

\begin{lemma}\label{lemma:initialization_norms} Suppose that $d  = \Omega(n(p+s) \log(nm / \delta) )$, $ m = \Omega(\log(1 / \delta))$. Then with probability at least $1 - \delta$, 
\begin{align*}
    | \langle \mathbf{w}_{j,r}^{(0)}, \Tilde{\boldsymbol{\xi}}_i \rangle | & \leq 4 \sqrt{ \log(8mn/\delta)}\cdot \sigma_0 \sigma_p \sqrt{d/(n(p+s))}, \\
    \sigma_0 \sigma_p \sqrt{d/(n(p+s))} / 4 & \leq \max_{r\in[m]} j\cdot \langle \mathbf{w}_{j,r}^{(0)}, \tilde{\boldsymbol{\xi}}_i \rangle \leq 2\sqrt{ \log(8mn/\delta)} \cdot \sigma_0 \sigma_p \sqrt{d/(n(p+s))},
\end{align*}
for all $j\in \{\pm 1\}$ and $i\in [n]$.
\end{lemma}
\begin{proof}[Proof of Lemma~\ref{lemma:initialization_norms}]

According to the fact that the weight $\mathbf{w}_{j,r}(0)$ and noise vector $\boldsymbol{\xi}$ are sampled from Gaussian distribution, we know that $ \langle \mathbf{w}_{j,r}^{(0)}, \Tilde{\boldsymbol{\xi}}_i \rangle$ is also Gaussian. By Lemma~\ref{lemma:graph_data_innerproducts}, with probability at least $1 - \delta / 4$, we have that
\begin{align*}
    \sigma_p \sqrt{d/(n(p+s))} / \sqrt{2} \leq \| \Tilde{\boldsymbol{\xi}}_i \|_2 \leq \sqrt{3/2}\cdot \sigma_p \sqrt{d/(n(p+s))}
\end{align*} 
holds for all $i\in [n]$. Therefore, applying the concentration bound for Gaussian variable, we obtain that 
\begin{align*}
    | \langle \mathbf{w}_{j,r}^{(0)}, \Tilde{\boldsymbol{\xi}}_i \rangle | \leq 4 \sqrt{ \log(8mn/\delta)}\cdot \sigma_0 \sigma_p \sqrt{d/(n(p+s))}.
\end{align*}
Next we finish the argument for the lower bound of maximum through the follow expression:
\begin{align*}
    P( \max \langle \mathbf{w}_{j,r}^{(0)}, \Tilde{\boldsymbol{\xi}}_i \rangle \ge \sigma_0 \sigma_p \sqrt{d/(n(p+s))} / 4  ) & =  1- P( \max \langle \mathbf{w}_{j,r}^{(0)}, \Tilde{\boldsymbol{\xi}}_i \rangle  < \sigma_0 \sigma_p \sqrt{d/(n(p+s))} / 4  ) \\
     & = 1 -  P( \max \langle \mathbf{w}_{j,r}^{(0)}, \Tilde{\boldsymbol{\xi}}_i \rangle  < \sigma_0 \sigma_p \sqrt{d/(n(p+s))} / 4 )^{2m} \\
     & \ge 1 - \delta/4.
\end{align*}
Together with Lemma \ref{lemma:graph_data_innerproducts}, we finally obtain that
\begin{align*}
    \sigma_0 \sigma_p \sqrt{d/(n(p+s))} / 4 & \leq \max_{r\in[m]} j\cdot \langle \mathbf{w}_{j,r}^{(0)}, \tilde{\boldsymbol{\xi}}_i \rangle \leq 2\sqrt{ \log(8mn/\delta)} \cdot \sigma_0 \sigma_p \sqrt{d/(n(p+s))}.
\end{align*}
\end{proof}

\section{General Lemmas for Iterative Coefficient Analysis} \label{seca:general}

In this section, we deliver lemmas that delineate the iterative behavior of coefficients under gradient descent. We commence with proving the coefficient update rules as stated in Lemma \ref{lemma:coefficient_iterative} in Section \ref{sec:weight_decop}. Subsequently, we establish the scale of training dynamics in Section \ref{sec:end_dynamics}.

\subsection{Coefficient update rule} \label{sec:weight_decop}

\begin{lemma}[Restatement of Lemma~\ref{lemma:coefficient_iterative}]\label{lemma:coefficient_iterative_proof}
The coefficients $\gamma_{j,r}^{(t)},\zeta_{j,r,i}^{(t)},\omega_{j,r,i}^{(t)}$ defined in Eq.~(\ref{eq:w_decomposition}) satisfy the following iterative equations:
\begin{align*}
&\gamma_{j,r}^{(0)},\overline{\rho}_{j,r,i}^{(0)},\underline{\rho}_{j,r,i}^{(0)} = 0,\\
    &\gamma_{j,r}^{(t+1)} = \gamma_{j,r}^{(t)} - \frac{\eta}{nm} \cdot \sum_{i=1}^n   {\ell}_i'^{(t)}   \sigma'(\langle \mathbf{w}_{j,r}^{(t)},  \Tilde{y}_i   {\boldsymbol{\mu} } \rangle) \nonumber    y_i \Tilde{y}_i   \| \boldsymbol{\mu} \|_2^2, \\
    & \overline{\rho}_{j,r,i}^{(t+1)}   = \overline{\rho}_{j,r,i}^{(t)} - \frac{\eta}{nm} \cdot \sum_{k \in \mathcal{N}(i)} D^{-1}_{k}  \cdot   {\ell}_k'^{(t)}   \cdot \sigma' (  \langle \mathbf{w}_{j,r}^{(t)},  \tilde{\boldsymbol{\xi}}_{k}   \rangle  )    \cdot \| \boldsymbol{\xi}_i \|_2^2 \cdot \mathds{1} (y_{k} = j), \\
   & \underline{\rho}_{j,r,i}^{(t+1)}   = \underline{\rho}_{j,r,i}^{(t)} - \frac{\eta}{nm} \cdot \sum_{k \in \mathcal{N}(i)} D^{-1}_{k}  \cdot  {\ell}_k'^{(t)}  \cdot \sigma' (   \langle \mathbf{w}_{j,r}^{(t)},  \tilde{\boldsymbol{\xi}}_{k}   \rangle  )     \cdot \| \boldsymbol{\xi}_i \|_2^2 \cdot \mathds{1} (y_{k} = -j), 
\end{align*} 
for all $r\in [m]$,  $j\in \{\pm 1\}$ and $i\in [n]$.
\end{lemma}
\begin{remark}
    {This lemma serves as a foundational element in our analysis of dynamics. Initially, the study of neural network dynamics under gradient descent required us to monitor the fluctuations in weights. However, this Lemma enables us to observe these dynamics through a new lens, focusing on two distinct aspects: signal learning and noise memorization. These are represented by the variables $\gamma^{(t)}_{j,r}$ and $\rho^{(t)}_{j,r,i}$, respectively. Furthermore, the selection of our data model was a conscious decision, designed to clearly separate the signal learning from the noise memorization aspects of learning. By maintaining a clear distinction between signal and noise, we can conduct a precise analysis of how each model learns the signal and memorizes the noise. This approach not only simplifies our understanding but also enhances our ability to dissect the underlying mechanisms of learning.}
\end{remark}

\begin{proof}[ {Proof of Lemma~\ref{lemma:coefficient_iterative_proof}}]  {Basically, the iteration of coefficients is derived based on gradient descent rule \eqref{eq:gcn_gdupdate} and weight decomposition \eqref{eq:w_decomposition}}. We first consider  $\hat \gamma_{j,r}^{(0)},\hat\rho_{j,r,i}^{(0)} = 0$ and 
\begin{align*}
    &\hat \gamma_{j,r}^{(t+1)} = \hat \gamma_{j,r}^{(t)} - \frac{\eta}{nm} \cdot \sum_{i=1}^n   {\ell}_i'^{(t)}   \sigma'(\langle \mathbf{w}_{j,r}^{(t)},  \Tilde{y}_i   {\boldsymbol{\mu}_i } \rangle) \nonumber    y_i \Tilde{y}_i   \| \boldsymbol{\mu} \|_2^2,\\
    &\hat{\rho}_{j,r,i}^{(t+1)}   = \hat{\rho}_{j,r,i}^{(t)} - \frac{\eta}{nm} \cdot \sum_{k \in \mathcal{N}(i)} D^{-1}_{k}  \cdot   {\ell}_k'^{(t)}   \cdot \sigma' (  \langle \mathbf{w}_{j,r}^{(t)},  \tilde{\boldsymbol{\xi}}_{k}   \rangle  )   \cdot \| \boldsymbol{\xi}_i \|_2^2 \cdot   y_{k},
\end{align*}
Taking above equations into Equation \eqref{eq:gcn_gdupdate}, we can obtain that 
\begin{align*}
   \mathbf{w}_{j,r}^{(t)} = \mathbf{w}_{j,r}^{(0)} + j \cdot \hat \gamma_{j,r}^{(t)} \cdot \| \boldsymbol{\mu} \|_2^{-2} \cdot \boldsymbol{\mu} + \sum_{ i = 1}^n \hat \rho_{j,r,i}^{(t)} \| \boldsymbol{\xi}_i \|_2^{-2} \cdot \boldsymbol{\xi}_{i}.
\end{align*}
{This result verifies that the iterative update of the coefficients is directly driven by the gradient descent update process. Furthermore, the uniqueness of the decomposition leads us to the precise relationships $\gamma_{j,r}^{(t)} = \hat \gamma_{j,r}^{(t)}$ and $\rho_{j,r,i}^{(t) } = \hat \rho_{j,r,i}^{(t)}$. Next, we examine the stability of the sign associated with noise memorization by employing the following telescopic analysis. This method allows us to investigate the continuity and consistency of the noise memorization process, providing insights into how the system behaves over successive iterations.}
\begin{align*}
\rho_{j,r,i}^{(t)} = - \sum_{s=0}^{t-1} \sum_{k \in \mathcal{N}(i)} D^{-1}_{k}  \frac{\eta}{nm} \cdot \ell_k'^{(s)}\cdot \sigma'(\langle \mathbf{w}_{j,r}^{(s)}, \tilde{\boldsymbol{\xi}}_{k}\rangle) \cdot \| \boldsymbol{\xi}_i \|_2^2 \cdot jy_{k} .   
\end{align*}
Recall the sign of loss derivative is given by the definition of the cross-entropy loss, namely, $\ell_i'^{(t)} < 0$. Therefore,
\begin{align}
&\overline{\rho}_{j,r,i}^{(t)} = - \sum_{s=0}^{t-1}\frac{\eta}{nm} \cdot \sum_{k \in \mathcal{N}(i)} D^{-1}_{k}  \cdot   {\ell}_k'^{(t)}   \cdot \sigma' (  \langle \mathbf{w}_{j,r}^{(t)},  \tilde{\boldsymbol{\xi}}_{k}   \rangle  )   \cdot \| \boldsymbol{\xi}_i \|_2^2 \cdot     \mathds{1}(y_{k} = j),\label{eq:iterate1}\\
& \underline{\rho}_{j,r,i}^{(t)} =  - \sum_{s=0}^{t-1}\frac{\eta}{nm} \cdot \sum_{k \in \mathcal{N}(i)} D^{-1}_{k}  \cdot   {\ell}_k'^{(t)}   \cdot \sigma' (  \langle \mathbf{w}_{j,r}^{(t)},  \tilde{\boldsymbol{\xi}}_{k}   \rangle  )   \cdot \| \boldsymbol{\xi}_i \|_2^2 \cdot     \mathds{1}(y_{k} = -j).\label{eq:iterate2}
\end{align}
Writing out the iterative versions of \eqref{eq:iterate1} and \eqref{eq:iterate2} completes the proof.
\end{proof}

\begin{remark}
    {The proof strategy follows the study of feature learning in CNN as described in \cite{cao2022benign}. However, compared to CNNs, the decomposition of weights in GNN is notably more intricate. This complexity is particularly evident in the dynamics of noise memorization, as represented by Equations \ref{eq:iterate1}) and \ref{eq:iterate2}). The reason for this increased complexity lies in the additional graph convolution operations within GNNs. These operations introduce new interaction and dependencies, making the analysis of weight dynamics more challenging and nuanced.}  
\end{remark}

\subsection{Scale of training dynamics} \label{sec:end_dynamics}

Our proof hinges on a meticulous evaluation of the coefficient values $\gamma_{j,r}^{(t)},\overline{\rho}_{j,r,i}^{(t) }, \underline{\rho}_{j,r,i}^{(t)}$ throughout the entire training process. In order to facilitate a more thorough analysis, we first establish the following bounds for these coefficients, which are maintained consistently throughout the training period.

{Consider training the Graph Neural Network (GNN) for an extended period up to $T^\ast$. We aim to investigate the scale of noise memorization in relation to signal learning.} 

Let $T^{*} = \eta^{-1}\text{poly}(\epsilon^{-1}, \|\boldsymbol{\mu}\|_{2}^{-1}, d^{-1}\sigma_{p}^{-2},\sigma_{0}^{-1}, n, m, d)$ be the maximum admissible iterations. Denote $\alpha = 4\log(T^{*})$. {In preparation for an in-depth analysis, we enumerate the necessary conditions that must be satisfied. These conditions, which are essential for the subsequent examination, are also detailed in Condition~\ref{condition:d_sigma0_eta}:}
\begin{align}
&\eta = O\Big(\min\{nm/(q\sigma_{p}^{2}d), nm/(q2^{q+2}\alpha^{q-2} \sigma_{p}^{2}d),  nm/(q2^{q+2}\alpha^{q-2}\|\boldsymbol{\mu}\|_{2}^{2})\}\Big), \label{eq: verify}\\
&\sigma_{0} \leq [16\sqrt{ \log(8mn/\delta)}]^{-1}\min \left \{ \Xi^{-1} \|\boldsymbol{\mu}\|_{2}^{-1}, (\sigma_{p}\sqrt{d/(n(p+s)) } )^{-1} \right \}, \label{eq: verifyy}\\
&d \geq 1024\log(4n^{2}/\delta)\alpha^{2}n^{2}. \label{eq: verifyyy}
\end{align}
Denote $\beta = 2 \max_{i,j,r}\{|\langle \mathbf{w}_{j,r}^{(0)}, \tilde{y}_i \cdot \boldsymbol{\mu} \rangle|,|\langle \mathbf{w}_{j,r}^{(0)}, \tilde{\boldsymbol{\xi}}_{i}\rangle|\}$, it is straightforward to show the following inequality: 
\begin{align}
4\max\bigg\{\beta, 8n\sqrt{\frac{\log(4n^{2}/\delta)}{d}}\alpha\bigg\} \leq 1. \label{eq:verify0}
\end{align}
First, by Lemma~\ref{lemma:graph_numberofdata} with probability at least $1- \delta$, we can upper bound $\beta$ by $4\sqrt{ \log(8mn/\delta)} \cdot \sigma_0 \cdot \max\{ \Xi \|\boldsymbol{\mu}\|_{2}, \sigma_p \sqrt{d/(n(p+s)) } \}$. Combined with the condition \eqref{eq: verifyy}, we can bound $\beta$ by $1$. Second, it is easy to check that $8n\sqrt{\frac{\log(4n^{2}/\delta)}{d}}\alpha \le 1$ by inequality \eqref{eq: verifyyy}.

{Having established the values of $\alpha$ and $\beta$ at hand, we are now in a position to assert that the following proposition holds for the entire duration of the training process, specifically for $0 \leq t \leq T^*$.}

\begin{proposition} \label{Prop:noise} 
Under Condition~\ref{condition:d_sigma0_eta}, for $0 \leq t\leq T^*$, where $T^{*} = \eta^{-1}\mathrm{poly}(\epsilon^{-1}, \|\boldsymbol{\mu}\|_{2}^{-1}, d^{-1}\sigma_{p}^{-2},\sigma_{0}^{-1}, n, m, d)$, we have that
\begin{align}
 &0\leq \gamma_{j,r}^{(t)}, \overline{\rho}_{j,r,i}^{(t)} \leq \alpha,\label{eq:gu0001}\\
&0\geq \underline{\rho}_{j,r,i}^{(t)} \geq  -\alpha\label{eq:gu0002},
\end{align}
for all $r\in [m]$,  $j\in \{\pm 1\}$ and $i\in [n]$, where $\alpha = 4\log(T^{*})$.
\end{proposition}
To establish Proposition~\ref{Prop:noise}, we will employ an inductive approach. Before proceeding with the proof, we need to introduce several technical lemmas that are fundamental to our argument.

{We note that although the setting is slightly different from the case in \cite{cao2022benign}. With the same analysis, we can obtain the following result.}
\begin{lemma} [{\cite{cao2022benign}}] \label{lm: mubound}
For any $t\geq 0$, it holds that $\langle \mathbf{w}_{j,r}^{(t)} - \mathbf{w}_{j,r}^{(0)}, \boldsymbol{\mu} \rangle =  j\cdot\gamma_{j,r}^{(t)}$ for all $r\in [m]$,  $j\in \{\pm 1\}$.
\end{lemma}

{In the subsequent three lemmas, our proof strategy is guided by the approach found in \cite{cao2022benign}. However, we extend this methodology by providing a fine-grained analysis that takes into account the additional complexity introduced by the graph convolution operation.}

\begin{lemma}\label{lm:oppositebound}
Under Condition~\ref{condition:d_sigma0_eta},  suppose \eqref{eq:gu0001} and \eqref{eq:gu0002} hold at iteration $t$. Then
\begin{align*}
\hat{\rho}_{j,r,i}^{(t)} - 8n\sqrt{\frac{\log(4n^{2}/\delta)}{d}}\alpha\leq \langle \mathbf{w}_{j,r}^{(t)} -  \mathbf{w}_{j,r}^{(0)}, \tilde{\boldsymbol{\xi}}_{i} \rangle &\leq \hat{\rho}_{j,r,i}^{(t)} + 8 n\sqrt{\frac{\log(4n^{2}/\delta)}{d}}\alpha, 
\end{align*}
where $\hat{\rho}_{j,r,i} \triangleq \sum_{k \in \mathcal{N} (i)} D^{-1}_i \sum_{i'\not= k} \rho_{j,r,i'}^{(t)} $, for all $r\in [m]$,  $j\in \{\pm 1\}$ and $i\in [n]$.
\end{lemma}

\begin{remark}
   {Lemma \ref{lm:oppositebound} asserts that the inner product between the updated weight and the graph convolution operation closely approximates the graph-convoluted noise memorization.}
\end{remark}

\begin{proof}[Proof of Lemma~\ref{lm:oppositebound}]
It is known that,
\begin{align*}
\langle \mathbf{w}_{j,r}^{(t)} - \mathbf{w}_{j,r}^{(0)} , \tilde{\boldsymbol{\xi}}_{i} \rangle &= \sum_{ i'= 1 }^n \zeta_{j,r,i'}^{(t)}  \| \boldsymbol{\xi}_{i'} \|_2^{-2} \cdot \langle \boldsymbol{\xi}_{i'}, \tilde{\boldsymbol{\xi}}_i \rangle + \sum_{ i' = 1}^n \omega_{j,r,i'}^{(t)}  \| \boldsymbol{\xi}_{i'} \|_2^{-2} \cdot \langle \boldsymbol{\xi}_{i'}, \tilde{\boldsymbol{\xi}}_i \rangle \\
& = \sum_{ i'= 1 }^n \sum_{k \in \mathcal{N}(i)} D^{-1}_i  \zeta^{(t)}_{j,r,i'}  \| \boldsymbol{\xi}_{i'} \|_2^{-2} \cdot \langle \boldsymbol{\xi}_{i'}, \boldsymbol{\xi}_k \rangle + \sum_{ i'= 1 }^n \sum_{k \in \mathcal{N}(i)} D^{-1}_i  \omega^{(t)}_{j,r,i'}  \| \boldsymbol{\xi}_{i'} \|_2^{-2} \cdot \langle \boldsymbol{\xi}_{i'}, \boldsymbol{\xi}_k \rangle   \\
&\leq  4\sqrt{\frac{\log(4n^{2}/\delta)}{d}}   \sum_{k \in \mathcal{N} (i)} D^{-1}_i \sum_{i'\not= k}|\zeta_{j,r,i'}^{(t)}|   +4\sqrt{\frac{\log(4n^{2}/\delta)}{d}}  \sum_{k \in \mathcal{N} (i)} D^{-1}_i \sum_{i'\not= k} |\omega_{j,r,i'}^{(t)}| \\
 & \quad  + \sum_{k \in \mathcal{N} (i)} D^{-1}_i \sum_{i'\not= k} \zeta_{j,r,i'}^{(t)} +  \sum_{k \in \mathcal{N} (i)} D^{-1}_i \sum_{i'\not= k} \omega_{j,r,i'}^{(t)}  \\
&\leq \hat{\rho}_{j,r,i}^{(t)} + 8 n\sqrt{\frac{\log(4n^{2}/\delta)}{d}}\alpha,
\end{align*}
where we define $\hat{\rho}_{j,r,i} \triangleq \sum_{k \in \mathcal{N} (i)} D^{-1}_i \sum_{i'\not= k} \rho_{j,r,i'}^{(t)} $ the second inequality is by Lemma~\ref{lemma:data_innerproducts} and the last inequality is by $|\zeta^{(t)}_{j,r,i'}|, |\omega_{j,r,i'}^{(t)}| \leq \alpha$ in \eqref{eq:gu0001}.

Similarly, we can show that:
\begin{align*}
\langle \mathbf{w}_{j,r}^{(t)} - \mathbf{w}_{j,r}^{(0)} , \tilde{\boldsymbol{\xi}}_{i} \rangle &= \sum_{ i'= 1 }^n \zeta_{j,r,i'}^{(t)}  \| \boldsymbol{\xi}_{i'} \|_2^{-2} \cdot \langle \boldsymbol{\xi}_{i'}, \tilde{\boldsymbol{\xi}}_i \rangle   + \sum_{ i' = 1}^n \omega_{j,r,i'}^{(t)}  \| \boldsymbol{\xi}_{i'} \|_2^{-2} \cdot \langle \boldsymbol{\xi}_{i'}, \tilde{\boldsymbol{\xi}}_i \rangle \\
& = \sum_{ i'= 1 }^n \sum_{k \in \mathcal{N}(i)} D^{-1}_i  \zeta^{(t)}_{j,r,i'}  \|\boldsymbol{\xi}_{i'} \|_2^{-2} \cdot \langle \boldsymbol{\xi}_{i'}, \boldsymbol{\xi}_k \rangle  + \sum_{ i'= 1 }^n \sum_{k \in \mathcal{N}(i)} D^{-1}_i  \omega^{(t)}_{j,r,i'}  \| \boldsymbol{\xi}_{i'} \|_2^{-2} \cdot \langle \boldsymbol{\xi}_{i'}, \boldsymbol{\xi}_k \rangle   \\
&\geq  -4\sqrt{\frac{\log(4n^{2}/\delta)}{d}}   \sum_{k \in \mathcal{N} (i)} D^{-1}_i \sum_{i'\not= k}|\zeta_{j,r,i'}^{(t)}|   - 4\sqrt{\frac{\log(4n^{2}/\delta)}{d}}  \sum_{k \in \mathcal{N} (i)} D^{-1}_i \sum_{i'\not= k} |\omega_{j,r,i'}^{(t)}| \\
 & \quad  + \sum_{k \in \mathcal{N} (i)} D^{-1}_i \sum_{i'\not= k} \zeta_{j,r,i'}^{(t)} +  \sum_{k \in \mathcal{N} (i)} D^{-1}_i \sum_{i'\not= k} \omega_{j,r,i'}^{(t)}  \\
&\geq \hat{\rho}_{j,r,i}^{(t)} - 8 n\sqrt{\frac{\log(4n^{2}/\delta)}{d}}\alpha,
\end{align*}
where the first inequality is by Lemma~\ref{lemma:data_innerproducts} and the second inequality is by $|\zeta^{(t)}_{j,r,i'}|, |\omega_{j,r,i'}^{(t)}| \leq \alpha$ in \eqref{eq:gu0001}, which completes the proof.
\end{proof}

\begin{lemma}\label{lm: F-yi} Under Condition~\ref{condition:d_sigma0_eta}, suppose \eqref{eq:gu0001} and \eqref{eq:gu0002} hold at iteration $t$. Then 
\begin{align*}
\langle \mathbf{w}_{j,r}^{(t)}, \Tilde{y}_{i}\boldsymbol{\mu} \rangle &\leq \langle \mathbf{w}_{j,r}^{(0)},   \Tilde{y}_{i}\boldsymbol{\mu} \rangle, \\
\langle \mathbf{w}_{j,r}^{(t)},   \Tilde{\boldsymbol{\xi}}_{i} \rangle &\leq \langle \mathbf{w}_{j,r}^{(0)},  \Tilde{\boldsymbol{\xi}}_{i}\rangle + 8n\sqrt{\frac{\log(4n^{2}/\delta)}{d}}\alpha,
\end{align*}
for all $r\in [m]$ and $j \not= y_{i}$. If $\max\{\gamma_{j,r}^{(t)},  {\rho}_{j,r,i}^{(t)}\} = O(1)$, we further have that $ {F}_{j}(\mathbf{W}_{j}^{(t)}, \Tilde{\mathbf{x}}_{i}) = O(1)$.
\end{lemma}

\begin{remark}
   {Lemma \ref{lm: F-yi} further establishes that the update in the direction of $\Tilde{\boldsymbol{\xi}}$ can be constrained within specific bounds when $j \neq y_i$. As a result, the output function remains controlled and does not exceed a constant order. }
\end{remark}

\begin{proof}[Proof of Lemma~\ref{lm: F-yi}]
For $j \not= y_{i}$, we have that 
\begin{align}
\langle \mathbf{w}_{j,r}^{(t)},  \Tilde{y}_{i}\boldsymbol{\mu} \rangle = \langle \mathbf{w}_{j,r}^{(0)},  \Tilde{y}_{i} \boldsymbol{\mu} \rangle + \Tilde{y}_{i}\cdot j \cdot \gamma_{j,r}^{(t)} \leq \langle \mathbf{w}_{j,r}^{(0)}, \tilde{y}_{i} \boldsymbol{\mu} \rangle, \label{eq:F-yi1}
\end{align}
where the inequality is by $\gamma_{j,r}^{(t)} \geq 0$ and Lemma \ref{lemma:graph_numberofdata} stating that $\mathrm{sign}(y_i) = \mathrm{sign}(\Tilde{y}_i) $ with a high probability. 
In addition, we have
\begin{align}
\langle \mathbf{w}_{j,r}^{(t)}, \Tilde{\boldsymbol{\xi}}_{i}  \rangle  & =  \langle \mathbf{w}_{j,r}^{(0)}, \Tilde{\boldsymbol{\xi}}_{i}\rangle  + \sum_{k \in \mathcal{N}(i)}   D^{-1}_i \sum_{i'=1}^n \rho_{j,r,i'} \langle \boldsymbol{\xi}_k, \boldsymbol{\xi}_{i'} \rangle  \| \boldsymbol{\xi}_{i'} \|^{-2}_2 \nonumber  \\
&\leq \langle \mathbf{w}_{j,r}^{(0)}, \Tilde{\boldsymbol{\xi}}_{i}\rangle  +  D^{-1}_i \left( \sum_{y_k \neq j} \omega_{j,r,i}^{(t)} + \sum_{y_k = j} \zeta_{j,r,i}^{(t)} \right) + 8n\sqrt{\frac{\log(4n^{2}/\delta)}{d}}\alpha \nonumber \\
& \leq \langle \mathbf{w}_{j,r}^{(0)}, \tilde{\boldsymbol{\xi}}_{i}\rangle + 8n\sqrt{\frac{\log(4n^{2}/\delta)}{d}}\alpha, \label{eq:F-yi2}
\end{align}
where the first inequality is by Lemma~\ref{lm:oppositebound} and the second inequality is due to $\hat{\rho}_{j,r,i}^{(t)} \leq 0$ based on Lemma \ref{lemma:graph_numberofdata}.
Then we can get that 
\begin{align*}
 {F}_{j}(\mathbf{W}_{j}^{(t)}, \Tilde{\mathbf{x}}_{i}) &= \frac{1}{m}\sum_{r=1}^{m}[\sigma(\langle \mathbf{w}_{j,r}^{(t)}, \Tilde{y}_i \cdot \boldsymbol{\mu} \rangle) + \sigma(\langle \mathbf{w}_{j,r}^{(t)} , \Tilde{\boldsymbol{\xi}}_{i} \rangle)]\\
& =  \frac{1}{m}\sum_{r=1}^{m}[\sigma(\langle \mathbf{w}_{j,r}^{(t)}, \Tilde{y}_i \cdot \boldsymbol{\mu} \rangle) + \sigma(\langle \mathbf{w}_{j,r}^{(t)} , D^{-1}_i \sum_{k \in \mathcal{N}(i)} \boldsymbol{\xi}_{k} \rangle)]   \\
& =  \frac{1}{m}\sum_{r=1}^{m}[\sigma(\langle \mathbf{w}_{j,r}^{(0)}, \Tilde{y}_i \cdot \boldsymbol{\mu} \rangle ) + \sigma(\langle \mathbf{w}_{j,r}^{(0)} , \Tilde{\boldsymbol{\xi}}_i  \rangle + \langle \mathbf{w}_{j,r}^{(t)} -\mathbf{w}_{j,r}^{(0)},  D^{-1}_i \sum_{k \in \mathcal{N}(i)} \boldsymbol{\xi}_{k} \rangle )]     \\
& \le \frac{1}{m}\sum_{r=1}^{m}[\sigma(\langle \mathbf{w}_{j,r}^{(0)}, \Tilde{y}_i \cdot \boldsymbol{\mu} \rangle ) + \sigma(\langle \mathbf{w}_{j,r}^{(0)} , \Tilde{\boldsymbol{\xi}}_i  \rangle +8n\sqrt{\frac{\log(4n^{2}/\delta)}{d}}\alpha + \hat{\rho}^{(t)}_{j,r,i} )]     \\
&\leq 2^{q+1} \max_{j,r,i} \bigg\{|\langle \mathbf{w}_{j,r}^{(0)},  \Tilde{y}_i \cdot \boldsymbol{\mu} \rangle|, |\langle \mathbf{w}_{j,r}^{(0)}, \Tilde{\boldsymbol{\xi}}_{i}\rangle|,   8n\sqrt{\frac{\log(4n^{2}/\delta)}{d}}\alpha\bigg\}^{q}\\
&\leq 1,
\end{align*}
where the first inequality is by \eqref{eq:F-yi1}, \eqref{eq:F-yi2} and the second inequality is by \eqref{eq:verify0} and $\max\{\gamma_{j,r}^{(t)},  {\rho}_{j,r,i}^{(t)}\} = O(1)$.
\end{proof}

\begin{lemma}\label{lm: Fyi}
Under Condition~\ref{condition:d_sigma0_eta}, suppose \eqref{eq:gu0001} and \eqref{eq:gu0002} hold at iteration $t$. Then
\begin{align*}
\langle \mathbf{w}_{j,r}^{(t)},  \tilde{y}_{i}\boldsymbol{\mu} \rangle &= \langle \mathbf{w}_{j,r}^{(0)},  \tilde{y}_{i}\boldsymbol{\mu} \rangle + \gamma_{j,r}^{(t)}, \\
\langle \mathbf{w}_{j,r}^{(t)},  \tilde{\boldsymbol{\xi}}_{i} \rangle &\leq \langle \mathbf{w}_{j,r}^{(0)},  \tilde{\boldsymbol{\xi}}_{i}\rangle + \hat{\rho}_{j,r,i}^{(t)} +  8n\sqrt{\frac{\log(4n^{2}/\delta)}{d}}\alpha
\end{align*}
for all $r\in [m]$, $j = y_i$ and $i\in [n]$.
If $\max\{\gamma_{j,r}^{(t)},  {\rho}_{j,r,i}^{(t)}\} = O(1)$, we further have that $ {F}_{j}(\mathbf{W}_{j}^{(t)}, \Tilde{\mathbf{x}}_{i}) = O(1)$.
\end{lemma}

\begin{remark}
   {Lemma \ref{lm: Fyi} further establishes that the update in the direction of $\boldsymbol{\mu}$ and $\Tilde{ \boldsymbol{\xi}}$ can be constrained within specific bounds when $j= y_i$. As a result, the output function remains controlled and does not exceed a constant order with an additional condition. }
\end{remark}

\begin{proof}[Proof of Lemma~\ref{lm: Fyi}]
For $j = y_{i}$, we have that 
\begin{align}
\langle \mathbf{w}_{j,r}^{(t)}, \tilde{y}_{i} \boldsymbol{\mu} \rangle = \langle \mathbf{w}_{j,r}^{(0)}, \tilde{y}_{i} \boldsymbol{\mu} \rangle +  \gamma_{j,r}^{(t)}, \label{eq:Fyi1}
\end{align}
where the equation is by Lemma~\ref{lm: mubound}. We also have that
\begin{align}
\langle \mathbf{w}_{j,r}^{(t)}, \tilde{\boldsymbol{\xi}}_{i} \rangle \leq \langle \mathbf{w}_{j,r}^{(0)}, \tilde{\boldsymbol{\xi}}_{i}\rangle + \hat{\rho}_{j,r,i}^{(t)} + 8n\sqrt{\frac{\log(4n^{2}/\delta)}{d}}\alpha, \label{eq:Fyi2}
\end{align}
where the inequality is by Lemma~\ref{lm:oppositebound}. 
If $\max\{\gamma_{j,r}^{(t)}, \rho_{j,r,i}^{(t)}\} = O(1)$, we have following bound 
\begin{align*}
 {F}_{j}(\mathbf{W}_{j}^{(t)}, \Tilde{\mathbf{x}}_{i}) &= \frac{1}{m}\sum_{r=1}^{m}[\sigma(\langle \mathbf{w}_{j,r}^{(t)}, \Tilde{y}_i \cdot\boldsymbol{\mu}) + \sigma(\langle \mathbf{w}_{j,r}^{(t)} , \Tilde{\boldsymbol{\xi}}_{i} \rangle)]\\
&\leq 2\cdot 3^{q} \max_{j,r,i} \bigg\{\gamma_{j,r}^{(t)}, |\hat{\rho}_{j,r,i}^{(t)}|, |\langle \mathbf{w}_{j,r}^{(0)}, \Tilde{y}_i \cdot \boldsymbol{\mu}) \rangle|, |\langle \mathbf{w}_{j,r}^{(0)}, \Tilde{\boldsymbol{\xi}}_{i}\rangle|, 8n\sqrt{\frac{\log(4n^{2}/\delta)}{d}}\alpha\bigg\}^{q}\\
&= O(1),
\end{align*}
where $\hat{\rho}^{(t)}_{j,r,i} = \frac{1}{D_i} \sum_{k \in \mathcal{N}(i)} \overline{\rho}^{(t)}_{j,r,k} \mathds{1}(y_k = j) + \overline{\rho}^{(t)}_{j,r,k} \mathds{1}(y_k \neq j)  $, the first inequality is by \eqref{eq:Fyi1}, \eqref{eq:Fyi2}. Then the second inequality is by \eqref{eq:verify0} where $\beta =  2\max_{i,j,r}\{|\langle \mathbf{w}_{j,r}^{(0)},  \tilde{y}_i \cdot \boldsymbol{\mu} \rangle|,|\langle \mathbf{w}_{j,r}^{(0)}, \tilde{\boldsymbol{\xi}}_{i}\rangle|\} \le 1 $ and condition that $\max\{\gamma_{j,r}^{(t)},  {\rho}_{j,r,i}^{(t)}\} = O(1)$.
\end{proof}

{Equipped with Lemmas \ref{lm: mubound} - \ref{lm: Fyi}, we are now prepared to prove Proposition~\ref{Prop:noise}. These lemmas provide the foundational building blocks and insights necessary for our proof, setting the stage for a rigorous and comprehensive demonstration of the proposition}

\begin{proof}[Proof of Proposition~\ref{Prop:noise}]
{Following a similar approach to the proof found in \cite{cao2022benign}, we employ an induction method. This technique allows us to build our argument step by step, drawing on established principles and extending them to our specific context, thereby providing a robust and systematic demonstration.} 

At the initial time step $t = 0$, the outcome is clear since all coefficients are set to zero. 

Next, we hypothesize that there exists a time $\tilde{T}$ less that $T^\ast$ during which Proposition~\ref{Prop:noise} holds true for every moment within the range $0 \leq t \leq \tilde{T}-1$. Our objective is to show that this proposition remains valid at $t = \tilde{T}$.

We aim to validate that equation \eqref{eq:gu0002} is applicable at $t = \tilde{T}$, meaning that,
\begin{align*}
    \omega^{(t)}_{j,r,i} \geq -\beta - 16n\sqrt{\frac{\log(4n^{2}/\delta)}{d}}\alpha,
\end{align*}
{for the given parameters. It's important to note that $\omega_{j,r,i}^{(t)} = 0$ when $j = y_{i}$. So we only need to consider instances where $j \not= y_{i}$.}

1) Under condition 
\begin{align*}
  \omega_{j,r,i}^{(\tilde{T}-1)} \leq -0.5\beta - 8n\sqrt{\frac{\log(4n^{2}/\delta)}{d}}\alpha,  
\end{align*} 
Lemma~\ref{lm:oppositebound} leads us to the following relationships:
\begin{align*}
\langle \mathbf{w}_{j,r}^{(\tilde{T}-1)} , \tilde{y}_i {\boldsymbol{\mu}} \rangle \leq  \hat{\rho}_{j,r,i}^{(\tilde{T}-1)}  + \langle \mathbf{w}_{j,r}^{(0)}, \tilde{y}_i {\boldsymbol{\mu}} \rangle + 8n\sqrt{\frac{\log(4n^{2}/\delta)}{d}}\alpha \leq 0,
\end{align*}
and thus
\begin{align*}
\omega_{j,r,i}^{(\tilde{T})} &= \omega_{j,r,i}^{(\tilde{T}-1)} + \frac{\eta}{nm} \sum_{k} D_k^{-1} \cdot  {\ell}_k'^{(\tilde{T}-1)}\cdot \sigma'(\langle \mathbf{w}_{j,r}^{(\tilde{T}-1)}, \Tilde{\boldsymbol{\xi}}_{k}\rangle ) \cdot \mathds{1}(y_{k} = -j)\|\boldsymbol{\xi}_{i}\|_{2}^{2}\\
&=  \omega_{j,r,i}^{(\tilde{T}-1)} \geq -\beta - 16n\sqrt{\frac{\log(4n^{2}/\delta)}{d}}\alpha,
\end{align*}
with the final inequality being supported by the induction hypothesis.

{2) Given the condition \(\omega_{j,r,i}^{(\tilde{T}-1)} \geq  -0.5\beta - 8n\sqrt{\frac{\log(4n^{2}/\delta)}{d}}\alpha\), we can derive the following:
\begin{align*}
\omega_{j,r,i}^{(\tilde{T})} &= \omega_{j,r,i}^{(\tilde{T}-1)} + \frac{\eta}{nm} \cdot  \sum_{k \in \mathcal{N}(i)} D^{-1}_{k}  {\ell}_k'^{(\tilde{T}-1)}\cdot \sigma'(\langle \mathbf{w}_{j,r}^{(T-1)}, \Tilde{\boldsymbol{\xi}}_{k}\rangle) \cdot \mathds{1}(y_{k} = -j)\|\boldsymbol{\xi}_{i}\|_{2}^{2}\\
&\geq -0.5\beta - 8n\sqrt{\frac{\log(4n^{2}/\delta)}{d}}\alpha - O\bigg(\frac{\eta\sigma_{p}^{2}d}{nm}\bigg)\sigma'\bigg(0.5\beta + 8n\sqrt{\frac{\log(4n^{2}/\delta)}{d}}\alpha\bigg)\\
&\geq -0.5\beta - 8n\sqrt{\frac{\log(4n^{2}/\delta)}{d}}\alpha - O\bigg(\frac{\eta q\sigma_{p}^{2}d}{nm}\bigg)\bigg(0.5\beta + 8n\sqrt{\frac{\log(4n^{2}/\delta)}{d}}\alpha\bigg)\\
&\geq -\beta - 16n\sqrt{\frac{\log(4n^{2}/\delta)}{d}}\alpha, 
\end{align*}
where we apply the inequalities \({\ell}_i'^{(\tilde{T}-1)}\leq 1\) and \(\|\boldsymbol{\xi}_{i}\|_{2} = O(\sigma_{p}^{2}d)\), and use the conditions \(\eta = O\big(nm/(q\sigma_{p}^{2}d)\big)\) and \(0.5\beta + 8n\sqrt{\frac{\log(4n^{2}/\delta)}{d}}\alpha \leq 1\), as specified in \eqref{eq: verify}.}

Next, we aim to show that \eqref{eq:gu0001} is valid for \(t = \tilde{T}\). We can express:
\begin{align}
    | {\ell}_i'^{(t)}| &= \frac{1}{1 + \exp\{ y_i \cdot [ {F}_{+1}(\mathbf{W}_{+1}^{(t)},\Tilde{\mathbf{x}}_i) -  {F}_{-1}(\mathbf{W}_{-1}^{(t)},\Tilde{\mathbf{x}}_i)] \} }\notag\\
    & \leq \exp\{ -y_{i} \cdot [ {F}_{+1}(\mathbf{W}_{+1}^{(t)},\Tilde{\mathbf{x}}_i) -  {F}_{-1}(\mathbf{W}_{-1}^{(t)},\Tilde{\mathbf{x}}_i)]\}\notag\\
    & \leq \exp\{ - {F}_{y_{i}}(\mathbf{W}_{y_{i}}^{(t)},\Tilde{\mathbf{x}}_i) + 1 \}. \label{eq:logit}
\end{align}
with the last inequality being a result of Lemma~\ref{lm: F-yi}. Additionally, we recall the update rules for \(\gamma_{j,r}^{(t+1)}\) and \(\zeta_{j,r,i}^{(t+1)}\):
\begin{align*}
   \gamma_{j,r}^{(t+1)} &= \gamma_{j,r}^{(t)} - \frac{\eta}{nm} \cdot \sum_{i=1}^n  {\ell}_i'^{(t)} \cdot \sigma'(\langle \mathbf{w}_{j,r}^{(t)}, \tilde{y}_{i} \cdot  \boldsymbol{\mu} \rangle ) y_i \tilde{y}_i \| \boldsymbol{\mu} \|_{2}^{2},\\
    \zeta_{j,r,i}^{(t+1)} &= \zeta_{j,r,i}^{(t)} - \frac{\eta}{nm} \cdot \sum_{k \in \mathcal{N}(i)}  D^{-1}_k  {\ell}_k'^{(t)}\cdot \sigma'(\langle \mathbf{w}_{j,r}^{(t)}, \tilde{\boldsymbol{\xi}}_{k}\rangle ) \cdot \mathds{1}(y_{k} = j)\|\boldsymbol{\xi}_{i}\|_{2}^{2}.
\end{align*}
We define \(t_{j,r,i}\) as the final moment \(t < T^{*}\) when \(\zeta_{j,r,i}^{(t)} \leq 0.5 \alpha\).

{We can express \(\zeta_{j,r,i}^{(\tilde{T})}\) as follows:
\begin{align}
\zeta_{j,r,i}^{(\tilde{T})} &= \zeta_{j,r,i}^{(t_{j,r,i})} - \underbrace{\frac{\eta}{nm}  \cdot \sum_{k \in \mathcal{N}(i)} D^{-1}_k \cdot  {\ell}_k'^{(t_{j,r,i})}\cdot \sigma'(\langle \mathbf{w}_{j,r}^{(t_{j,r,i})}, \tilde{\boldsymbol{\xi}}_{k}\rangle) \cdot  \mathds{1}(y_{k} = j)\|\boldsymbol{\xi}_{i}\|_{2}^{2}}_{I_{1}}\notag\\
&\qquad - \underbrace{\sum_{t_{j,r,i}<t<T}\frac{\eta}{nm} \cdot \sum_{k \in \mathcal{N}(i)} D^{-1}_k \cdot  {\ell}_k'^{(t)}\cdot \sigma'(\langle \mathbf{w}_{j,r}^{(t)}, \tilde{\boldsymbol{\xi}}_{k}\rangle) \cdot  \mathds{1}(y_{k} = j)\|\boldsymbol{\xi}_{i}\|_{2}^{2}}_{I_{2}}.\label{eq:zeta}
\end{align}
Next, we aim to establish an upper bound for \(I_{1}\):
\begin{align*}
|I_{1}| & \leq 2qn^{-1}m^{-1}\eta  \bigg(\max_k \hat{\rho}_{j,r,k}^{(t_{j,r,i})} + 0.5\beta + 8n\sqrt{\frac{\log(4n^{2}/\delta)}{d}}\alpha\bigg)^{q-1}\sigma_{p}^{2}d  \\
& \leq  q2^{q}n^{-1}m^{-1}\eta \alpha^{q-1} \sigma_{p}^{2}d \leq 0.25\alpha,
\end{align*}
where we apply Lemmas~\ref{lm:oppositebound} and~\ref{lemma:data_innerproducts} for the first inequality, utilize the conditions \(\beta \leq 0.1\alpha\) and \(8n\sqrt{\frac{\log(4n^{2}/\delta)}{d}}\alpha \leq 0.1\alpha\) for the second inequality, and finally, the constraint \(\eta \leq nm/(q2^{q+2}\alpha^{q-2} \sigma_{p}^{2}d)\) for the last inequality.
}

Second, we bound $I_{2}$. For $t_{j,r,i}<t<\tilde{T}$ and $y_{k} = j$, we can lower bound $\langle \mathbf{w}_{j,r}^{(t)}, \tilde{\boldsymbol{\xi}}_{k}\rangle$ as follows, 
 \begin{align*}
\langle \mathbf{w}_{j,r}^{(t)}, \tilde{\boldsymbol{\xi}}_{k}\rangle &\geq \langle \mathbf{w}_{j,r}^{(0)},  \tilde{\boldsymbol{\xi}}_{k}\rangle + \hat{\rho}_{j,r,k}^{(t)} - 8n\sqrt{\frac{\log(4n^{2}/\delta)}{d}}\alpha \\
&\geq - 0.5\beta + \frac{1}{4} \frac{p-s}{p+s}\alpha - 8n\sqrt{\frac{\log(4n^{2}/\delta)}{d}}\alpha\\
&\geq 0.25\alpha, 
\end{align*}
where the first inequality is by Lemma~\ref{lm:oppositebound}, the second inequality is by $\hat{\rho}_{j,r,i}^{(t)} > \frac{1}{4} \frac{p-s}{p+s} \alpha$ and $\langle \mathbf{w}_{j,r}^{(0)},  \tilde{\boldsymbol{\xi}}_{i} \rangle \geq - 0.5\beta$ due to the definition of $t_{j,r,i}$ and $\beta$, the last inequality is by $\beta \leq 0.1\alpha$ and $8n\sqrt{\frac{\log(4n^{2}/\delta)}{d}}\alpha \leq 0.1\alpha$. Similarly, for $t_{j,r,i}<t<\tilde{T}$ and $y_{k} = j$, we can also upper bound $\langle \mathbf{w}_{j,r}^{(t)}, \tilde{\boldsymbol{\xi}}_{k}\rangle$ as follows, 
 \begin{align*}
\langle \mathbf{w}_{j,r}^{(t)}, \tilde{\boldsymbol{\xi}}_{k}\rangle &\leq \langle \mathbf{w}_{j,r}^{(0)},  \tilde{\boldsymbol{\xi}}_{k}\rangle + \hat{\rho}_{j,r,k}^{(t)} + 8n\sqrt{\frac{\log(4n^{2}/\delta)}{d}}\alpha \\
&\leq 0.5\beta + \frac{3}{4}\frac{p-s}{p+s}\alpha + 8n\sqrt{\frac{\log(4n^{2}/\delta)}{d}}\alpha\\
&\leq 2\alpha, 
\end{align*}
where the first inequality is by Lemma~\ref{lm:oppositebound}, the second inequality is by induction hypothesis $\hat{\rho}_{j,r,i}^{(t)} \leq \alpha$, the last inequality is by $\beta \leq 0.1\alpha$ and $8n\sqrt{\frac{\log(4n^{2}/\delta)}{d}}\alpha \leq 0.1\alpha$. 

{Hence, we can derive the following expression for \(I_{2}\):
\begin{align*}
|I_{2}| &\leq \sum_{t_{j,r,i}<t<\tilde{T}}\frac{\eta}{nm} \cdot  \sum_{k \in \mathcal{N}(i)} D^{-1}_k \exp(- \sigma(\langle \mathbf{w}_{j,r}^{(t)}, \Tilde{\boldsymbol{\xi}}_{k}\rangle) + 1)\cdot \sigma'(\langle \mathbf{w}_{j,r}^{(t)}, \tilde{\boldsymbol{\xi}}_{k} \rangle) \cdot  \mathds{1}(y_{k} = j)\|\boldsymbol{\xi}_{i}\|_{2}^{2}\\
&\leq \frac{eq2^{q}\eta T^{*}}{n}\exp(-\alpha^{q}/4^{q})\alpha^{q-1}\sigma_{p}^{2}d\\
&\leq 0.25 T^{*}\exp(-\alpha^{q}/4^{q})\alpha \\
&\leq 0.25 T^{*}\exp(-\log(T^{*})^{q})\alpha \\
&\leq 0.25\alpha,
\end{align*}
where we apply \eqref{eq:logit} for the first inequality, utilize Lemma~\ref{lemma:data_innerproducts} for the second, employ the constraint \(\eta = O\big( nm/(q2^{q+2}\alpha^{q-2} \sigma_{p}^{2}d)\big)\) in \eqref{eq: verify} for the third, and finally, the conditions \(\alpha = 4\log(T^{*})\) and \(\log(T^{*})^{q} \geq \log(T^{*})\) for the subsequent inequalities. By incorporating the bounds of \(I_{1}\) and \(I_{2}\) into \eqref{eq:zeta}, we conclude the proof for \(\zeta\).}

In a similar manner, we can establish that \(\gamma_{j,r}^{(\tilde{T})} \leq \alpha\) by using \(\eta = O\big( nm/(q2^{q+2}\alpha^{q-2}\|\boldsymbol{\mu}\|_{2}^{2})\big)\) in \eqref{eq: verify}. 
Thus, Proposition~\ref{Prop:noise} is valid for \(t= \tilde{T}\), completing the induction process. 
As a corollary to Proposition~\ref{Prop:noise}, we identify a crucial characteristic of the loss function during training within the interval \(0 \leq t\leq T^{*}\). This characteristic will play a vital role in the subsequent convergence analysis.

\end{proof}

\section{Two Stage Dynamics Analysis} \label{seca:dynamics}

In this section, we employ a two-stage dynamics analysis to investigate the behavior of coefficient iterations. During the first stage, the derivative of the loss function remains almost constant due to the small weight initialization. In the second stage, the derivative of the loss function ceases to be constant, necessitating an analysis that meticulously takes this into account. 

\subsection{First stage: feature learning versus noise memorization} \label{sec:stage_1}

\begin{lemma}[Restatement of Lemma~\ref{lemma:phase1_main_sketch}]\label{lemma:phase1_main}
Under the same conditions as Theorem~\ref{thm:signal_learning_main}, in particular if we choose
\begin{align}
n \cdot \mathrm{SNR}^{q} \cdot (n(p+s))^{q/2-1} \geq C\log(6/\sigma_{0}\|\boldsymbol{\mu}\|_{2})2^{2q+6}[4\log(8mn/\delta)]^{(q-1)/2},\label{eq:explicit condition}
\end{align}
where $C = O(1)$ is a positive constant, there exists time 
$T_1 = \frac{C\log(6/\sigma_{0}\|\boldsymbol{\mu}\|_{2})2^{q+1}m}{\eta\sigma_{0}^{q-2}\|\boldsymbol{\mu}\|_{2}^{q} \Xi^q}$
such that 
\begin{itemize}
\item $\max_{ r}\gamma_{j, r}^{(T_{1})} \geq 2$ for $j\in \{\pm 1\}$.
\item $|\rho_{j,r,i}^{(t)}| \leq \sigma_0 \sigma_p \sqrt{d/(n(p+s))} / 2$ for all $j\in \{\pm 1\}, r\in[m]$, $i \in [n]$ and $0 \leq t \leq T_{1}$. 
\end{itemize}
\end{lemma}

\begin{remark}
   {In this lemma, we establish that the rate of signal learning significantly outpaces that of noise memorization within GNNs. After a specific number of iterations, the GNN is able to learn the signal from the data at a constant or higher order, while only memorizing a smaller order of noise.}
\end{remark}

\begin{proof}[Proof of Lemma~\ref{lemma:phase1_main}] 

{Let us define
\begin{align}
    T_1^{+} = \frac{nm\eta^{-1}\sigma_{0}^{2-q}\sigma_{p}^{-q}d^{-q/2} {(n(p+s))^{(q-2)/2}}}{2^{q+4}q[4\log(8mn/\delta)]^{(q-2)/2}}. \label{eq:T1upper}
\end{align}
We will begin by establishing the outcome related to noise memorization. 
Let \(\Psi^{(t)}\) be the maximum value over all \(j,r,i\) of \(|\rho_{j,r,i}^{(t)}|\), that is, \(\Psi^{(t)} = \max_{j,r,i}\{ \overline{\rho}_{j,r,i}^{(t)},  -\underline{\rho}_{j,r,i}^{(t)}\}\). We will employ an inductive argument to demonstrate that
\begin{align}
    \Psi^{(t)} \leq \sigma_0 \sigma_p \sqrt{d/(n(p+s))}  \label{eq:Psi_induction}
\end{align}
is valid for the entire range \(0 \leq t \leq T_{1}^{+}\). By its very definition, it is evident that \(\Psi^{(0)} = 0\). Assuming that there exists a value \(\tilde{T} \leq T_1^+\) for which equation \eqref{eq:Psi_induction} is satisfied for all \(0 < t \leq \tilde{T}-1\), we can proceed as follows.}
\begin{align*}
    \Psi^{(t+1)}  & \le \Psi^{(t)} +   \frac{\eta}{nm}  \sum_{k \in \mathcal{N}(i)} D^{-1}_{k} \cdot | {\ell}_k'^{(t)}|\cdot  \\
    & \quad \sigma'\Bigg(\langle \mathbf{w}_{j,r}^{(0)}, \tilde{\boldsymbol{\xi}}_{k} \rangle + \sum_{ i'= 1 }^n \Psi^{(t)} \cdot \frac{ |\langle  \boldsymbol{\xi}_{i'}, \tilde{\boldsymbol{\xi}}_k \rangle|}{ \| \boldsymbol{\xi}_{i'} \|_2^2} + \sum_{ i' = 1}^n \Psi^{(t)} \cdot \frac{|\langle \boldsymbol{\xi}_{i'}, \tilde{\boldsymbol{\xi}}_k \rangle|}{ \|\boldsymbol{\xi}_{i'} \|_2^2} \Bigg)\cdot \| \boldsymbol{\xi}_{i} \|_2^2   \\
    &\leq \Psi^{(t)} +  \frac{\eta}{nm} \cdot  \sum_{k \in \mathcal{N}(i)} D^{-1}_{k} \sigma'\Bigg(\langle \mathbf{w}_{j,r}^{(0)}, \tilde{\boldsymbol{\xi}}_{k} \rangle + 2\cdot \sum_{ i'= 1 }^n \Psi^{(t)} \cdot \frac{ |\langle  {\boldsymbol{\xi}}_{i'}, \tilde{\boldsymbol{\xi}}_k \rangle| }{ \| \boldsymbol{\xi}_{i'} \|_2^2}  \Bigg)\cdot \| \boldsymbol{\xi}_{i} \|_2^2  \\
    & = \Psi^{(t)} +  \frac{\eta}{nm} \cdot  \sum_{k \in \mathcal{N}(i)} D^{-1}_{k} \cdot \\
     & \quad \sigma'\Bigg(\langle \mathbf{w}_{j,r}^{(0)}, \tilde{\boldsymbol{\xi}}_{k} \rangle + 2\Psi^{(t)}  + 2\cdot \sum_{ i' \neq k' }^n \Psi^{(t)} \cdot D^{-1}_k \sum_{k' \in \mathcal{N}(k)} \frac{ |\langle \boldsymbol{\xi}_{i'}, \boldsymbol{\xi}_{k'} \rangle| }{ \|\boldsymbol{\xi}_{i'} \|_2^2} \Bigg)\cdot \| \boldsymbol{\xi}_{i} \|_2^2  \\
    &\leq \Psi^{(t)} + \frac{\eta q}{nm} \cdot \sum_{k \in \mathcal{N}(i)} D^{-1}_{k}  \Bigg[2\cdot \sqrt{ \log(8mn/\delta)} \cdot \sigma_0 \sigma_p \sqrt{d/(n(p+s))}  \\
     & \quad  + \Bigg( 2 + \frac{4n  \sigma_p^2 \cdot \sqrt{d \log(4n^2 / \delta) }    }{ \sigma_p^2 d  } \Bigg) \cdot \Psi^{(t)}  \Bigg]^{q-1}\cdot 2 \sigma_p^2 d \\
    &\leq \Psi^{(t)} + \frac{\eta q}{nm} \cdot  \big(2\cdot \sqrt{ \log(8mn/\delta)} \cdot \sigma_0 \sigma_p \sqrt{d/(n(p+s))} + 4 \Psi^{(t)}  \big)^{q-1}\cdot 2 \sigma_p^2 d \\
    &\leq \Psi^{(t)} + \frac{\eta q}{nm} \cdot \big(4\cdot \sqrt{ \log(8mn/\delta)} \cdot \sigma_0 \sigma_p \sqrt{d/(n(p+s))}  \big)^{q-1} \cdot 2 \sigma_p^2 d,
 \end{align*}

{where the second inequality is due to the constraint \(| {\ell}_i'^{(t)}| \leq 1\), the third inequality is derived from Lemmas~\ref{lemma:data_innerproducts} and \ref{lemma:initialization_norms}, the fourth inequality is a consequence of the condition \(d \geq 16 D n^2 \log (4n^2/\delta)\), and the final inequality is a result of the inductive assumption \eqref{eq:Psi_induction}. Summing over the sequence \(t=0,1,\ldots, \tilde{T}-1\), we obtain
 \begin{align*}
    \Psi^{(\tilde{T})} 
    &\leq \tilde{T} \cdot\frac{\eta q}{nm} \cdot \big(4\cdot \sqrt{ \log(8mn/\delta)} \cdot \sigma_0 \sigma_p \sqrt{d/(n(p+s))}  \big)^{q-1} \cdot 2 \sigma_p^2 d \\
    &\leq T_{1}^+\cdot\frac{\eta q}{nm} \cdot \big(4\cdot \sqrt{ \log(8mn/\delta)} \cdot \sigma_0 \sigma_p \sqrt{d/(n(p+s))}  \big)^{q-1} \cdot 2 \sigma_p^2 d \\
     &\leq \frac{\sigma_0 \sigma_p \sqrt{d/(n(p+s))}   }{2},
\end{align*}
where the second inequality is justified by \(\tilde{T} \leq T_1^+\) in our inductive argument. Hence, by induction, we conclude that \( \Psi^{(t)} \leq \sigma_0 \sigma_p \sqrt{d/n(p+s)} / 2\) for all \(t \leq T_{1}^{+}\).}

Next, we can assume, without loss of generality, that \(j = 1\). Let \(T_{1,1}\) represent the final time for \(t\) within the interval \([0, T_1^{+}]\) such that \(\max_{r}\gamma_{1,r}^{(t)}\leq 2\), given \(\sigma_0 \le \sqrt{n(p+s)/d} /\sigma_p \). For \(t \leq T_{1,1}\), we have \(\max_{j,r,i}\{ |\rho_{j,r,i}^{(t)}|\} = O(\sigma_{0}\sigma_{p}\sqrt{d/(n(p+s))})= O(1)\) and \(\max_{r}\gamma_{1,r}^{(t)} \leq 2\). By applying Lemmas~\ref{lm: F-yi} and~\ref{lm: Fyi}, we deduce that \( {F}_{-1}(\mathbf{W}_{-1}^{(t)},\Tilde{\mathbf{x}}_{i}),  {F}_{+1}(\mathbf{W}_{+1}^{(t)},\Tilde{\mathbf{x}}_{i}) = O(1)\) for all \(i\) with \(y_{i} = 1\). Consequently, there exists a positive constant \(C_{1}\) such that \( -\ell'^{(t)}_{i} \geq C_{1}\) for all \(i\) with \(y_{i} = 1\).

By \eqref{eq:update_gamma1}, for $t\leq T_{1,1}$ we have
\begin{align*}
    \gamma_{1,r}^{(t+1)} &= \gamma_{1,r}^{(t)} - \frac{\eta}{nm} \cdot \sum_{i=1}^n {\ell}_i'^{(t)} \cdot \sigma'( \Tilde{y}_{i}  \cdot  \langle \mathbf{w}_{1,r}^{(0)}, \boldsymbol{\mu} \rangle +  \Tilde{y}_{i}   \cdot \gamma_{1,r}^{(t)} )\cdot \Tilde{y}_{i} \|\boldsymbol{\mu}\|_{2}^{2}\\
    &\geq \gamma_{1,r}^{(t)} + \frac{C_{1}\eta}{nm} \cdot \sum_{y_i=1} \sigma'( y_{i} \Xi \cdot  \langle \mathbf{w}_{1,r}^{(0)}, \boldsymbol{\mu} \rangle +  y_{i} \Xi  \cdot \gamma_{1,r}^{(t)} )\cdot \frac{p-s}{p+s} \|\boldsymbol{\mu}\|_{2}^{2}.
\end{align*}
Denote $\hat{\gamma}_{1,r}^{(t)} = \gamma_{1,r}^{(t)} + \langle \mathbf{w}_{1,r}^{(0)}, \boldsymbol{\mu} \rangle $ and let $A^{(t)} = \max_{r}\hat{\gamma}_{1,r}^{(t)}$. Then we have 
\begin{align*}
   A^{(t+1)}&\geq A^{(t)} + \frac{C_{1}\eta}{nm} \cdot \sum_{y_i=1} \sigma'(\Xi A^{(t)} )\cdot \Xi \|\boldsymbol{\mu}\|_{2}^{2}\\
    &\geq A^{(t)} +\frac{C_{1} \eta q\|\boldsymbol{\mu}\|_{2}^{2}}{4m}  \bigg[\Xi A^{(t)}\bigg]^{q-1} \Xi \\
    &\geq \bigg(1 + \frac{C_{1} \eta q\|\boldsymbol{\mu}\|_{2}^{2}}{4m}\big[A^{(0)}\big]^{q-2} \Xi^{q} \bigg) A^{(t)}\\
    &\geq \bigg(1 + \frac{C_{1}\eta q\sigma_{0}^{q-2}\|\boldsymbol{\mu}\|_{2}^{q}}{2^{q}m} \Xi^{q} \bigg) A^{(t)},
\end{align*}

{where the second inequality arises from the lower bound on the quantity of positive data as established in Lemma~\ref{lemma:graph_numberofdata}, the third inequality is a result of the increasing nature of the sequence \(A^{(t)}\), and the final inequality is derived from \(A^{(0)} = \max_{r} \langle \mathbf{w}_{1,r}^{(0)}, \boldsymbol{\mu} \rangle \geq \sigma_0 \|\boldsymbol{\mu}\|_2/2\), as proven in Lemma~\ref{lemma:initialization_norms}. Consequently, the sequence \(A^{(t)}\) exhibits exponential growth, and we can express it as}
\begin{align*}
   A^{(t)}&\geq\bigg(1 + \frac{C_{1}\eta q\sigma_{0}^{q-2}\|\boldsymbol{\mu}\|_{2}^{q}}{2^{q}m} \Xi^{q} \bigg)^{t} A^{(0)} \\
    & \geq \exp\bigg(\frac{C_{1}\eta q\sigma_{0}^{q-2}\|\boldsymbol{\mu}\|_{2}^{q}}{2^{q+1}m} \Xi^{q} t \bigg)A^{(0)} \\
    & \geq \exp\bigg(\frac{C_{1}\eta q\sigma_{0}^{q-2}\|\boldsymbol{\mu}\|_{2}^{q}}{2^{q+1}m} \Xi^{q} t\bigg)\frac{\sigma_{0}\|\boldsymbol{\mu}\|_{2}}{2},
\end{align*}
where the second inequality is justified by the relation \(1+z \geq \exp(z/2)\) for \(z \leq 2\) and our specific conditions on \(\eta\) and \(\sigma_{0}\) as listed in Condition~\ref{condition:d_sigma0_eta}. The last inequality is a consequence of Lemma~\ref{lemma:initialization_norms} and the definition of \(A^{(0)}\). Thus, \(A^{(t)}\) will attain the value of \(2\) within \(T_{1}\) iterations, defined as
\[
T_{1} = \frac{\log(6/\sigma_{0}\|\boldsymbol{\mu}\|_{2})2^{q+1}m}{C_{1}\eta q\sigma_{0}^{q-2}\|\boldsymbol{\mu}\|_{2}^{q}\Xi^{q} }.
\]
Since \(\max_{r}\gamma_{1,r}^{(t)} \geq A^{(t)} - 1\), \(\max_{r}\gamma_{1,r}^{(t)}\) will reach \(2\) within \(T_{1}\) iterations. Next, we can confirm that
\begin{align*}
T_{1} \leq \frac{nm\eta^{-1}\sigma_{0}^{2-q}\sigma_{p}^{-q}d^{-q/2} (n(p+s))^{(q-2)/2}} {2^{q+5}q[4\log(8mn/\delta)]^{(q-1)/2}} =  T_{1}^{+}/2,
\end{align*}
where the inequality is consistent with our SNR condition in \eqref{eq:explicit condition}. Therefore, by the definition of \(T_{1,1}\), we deduce that \(T_{1,1} \leq T_{1} \leq T_{1}^{+}/2\), utilizing the non-decreasing property of \(\gamma\). The proof for \(j=-1\) follows a similar logic, leading us to the conclusion that \(\max_{r}\gamma_{-1,r}^{(T_{1,-1})} \geq 2\) while \(T_{1,-1} \leq T_{1} \leq T_{1}^{+}/2\), thereby completing the proof.

\end{proof}

\subsection{Second stage: convergence analysis}

After the first stage and at time step $T_1$ we know that: 
\begin{align*}
\mathbf{w}_{j,r}^{(T_{1})} &= \mathbf{w}_{j,r}^{(0)} + j \cdot \gamma_{j,r}^{(T_{1})} \cdot \frac{\boldsymbol{\mu}}{\|\boldsymbol{\mu}\|_{2}^{2}} + \sum_{ i = 1}^n \zeta_{j,r,i}^{(T_{1})} \cdot \frac{\boldsymbol{\xi}_{i}}{\|\boldsymbol{\xi}_{i}\|_{2}^{2}} + \sum_{ i = 1}^n \omega_{j,r,i}^{(T_{1})} \cdot \frac{\boldsymbol{\xi}_{i}}{\|\boldsymbol{\xi}_{i}\|_{2}^{2}}.
\end{align*}
And at the beginning of the second stage, we have following property holds:
\begin{itemize}
\item $\max_{r}\gamma_{j, r}^{(T_{1})} \geq 2, \forall j \in \{\pm 1\}$. 
\item $\max_{j,r,i}|\rho_{j,r,i}^{(T_{1})}| \leq \hat{\beta}$ where $\hat{\beta} = \sigma_0 \sigma_p \sqrt{d/(n(p+s))} / 2$. 
\end{itemize}
Lemma~\ref{lemma:coefficient_iterative} implies that the learned feature $\gamma_{j,r}^{(t)}$ will not get worse, i.e., for $t \geq T_{1}$, we have that $\gamma_{j,r}^{(t+1)} \geq \gamma_{j,r}^{(t)} $, and therefore $\max_{ r}\gamma_{j, r}^{(t)} \geq 2$. Now we choose $\mathbf{W}^{*}$ as follows:
\begin{align*}
\mathbf{w}^{*}_{j,r} = \mathbf{w}_{j,r}^{(0)} + 2qm\log(2q/\epsilon) \cdot j \cdot  \frac{\boldsymbol{\mu}}{\|\boldsymbol{\mu}\|_{2}^{2}}. 
\end{align*}

{While the context of CNN presents subtle differences from the scenario described in CNN \cite{cao2022benign}, we can adapt the same analytical approach to derive the following two lemmas:}

\begin{lemma} [{\cite{cao2022benign}}]   \label{lm:distance1}
Under the same conditions as Theorem~\ref{thm:signal_learning_main}, we have that $\|\mathbf{W}^{(T_{1})} - \mathbf{W}^{*}\|_{F} \leq \tilde{O}(m^{3/2}\|\boldsymbol{\mu}\|_{2}^{-1})$.
\end{lemma}

\begin{lemma}  [{\cite{cao2022benign}}]   \label{lemma:signal_stage2_homogeneity}
Under the same conditions as Theorem~\ref{thm:signal_learning_main}, we have that  
\begin{align*}
\|\mathbf{W}^{(t)} - \mathbf{W}^{*}\|_{F}^{2} - \|\mathbf{W}^{(t+1)} - \mathbf{W}^{*}\|_{F}^{2} \geq (2q-1)\eta L_{\mathcal{S}}(\mathbf{W}^{(t)}) - \eta\epsilon
\end{align*}
for all $ T_{1} \leq t\leq T^{*}$.
\end{lemma}

\begin{lemma}[Restatement of Lemma~\ref{lemma:signal_proof_sketch}]\label{thm:signal_proof}
Under the same conditions as Theorem~\ref{thm:signal_learning_main}, let $T = T_{1} + \Big\lfloor \frac{\|\mathbf{W}^{(T_{1})} - \mathbf{W}^{*}\|_{F}^{2}}{2\eta \epsilon}
\Big\rfloor = T_{1} + \tilde{O}(m^{3}\eta^{-1}\epsilon^{-1}\|\boldsymbol{\mu}\|_{2}^{-2})$. Then we have $\max_{j,r,i}|\rho_{j,r,i}^{(t)}| \leq 2\hat{\beta} = \sigma_{0}\sigma_{p}\sqrt{d/(n(p+s))}$ for all $T_{1} \leq t\leq T$. Besides,
\begin{align*}
\frac{1}{t - T_{1} + 1}\sum_{s=T_{1}}^{t}L_{\mathcal{S}}(\mathbf{W}^{(s)}) \leq  \frac{\|\mathbf{W}^{(T_{1})} - \mathbf{W}^{*}\|_{F}^{2}}{(2q-1) \eta(t - T_{1} + 1)} + \frac{\epsilon}{2q-1} 
\end{align*}
for all $T_{1} \leq t\leq T$, and we can find an iterate with training loss smaller than $\epsilon$ within $T $ iterations.
\end{lemma}
\begin{proof}[Proof of Lemma~\ref{thm:signal_proof}]  {We adapt the convergence proof for CNN\cite{cao2022benign} to  extend the analysis to GNN. By invoking Lemma~\ref{lemma:signal_stage2_homogeneity}, for any given time interval \(t\in[T_1,T]\), we can deduce that
\begin{align*}
\|\mathbf{W}^{(s)} - \mathbf{W}^{*}\|_{F}^{2} - \|\mathbf{W}^{(s+1)} - \mathbf{W}^{*}\|_{F}^{2} \geq (2q-1)\eta L_{\mathcal{S}}(\mathbf{W}^{(s)}) - \eta\epsilon,
\end{align*}
which is valid for \(s \leq t\). Summing over this interval, we arrive at
\begin{align}
\sum_{s=T_{1}}^{t}L_{\mathcal{S}}(\mathbf{W}^{(s)}) &\leq \frac{\|\mathbf{W}^{(T_{1})} - \mathbf{W}^{*}\|_{F}^{2} + \eta\epsilon (t - T_{1} + 1)}{(2q-1) \eta}.\label{eq:vanillasum}
\end{align}
This inequality holds for all \(T_{1} \leq t\leq T\). 
Dividing both sides of \eqref{eq:vanillasum} by \((t-T_{1}+1)\), we obtain
\begin{align*}
\frac{1}{t - T_{1} + 1}\sum_{s=T_{1}}^{t}L_{\mathcal{S}}(\mathbf{W}^{(s)}) \leq  \frac{\|\mathbf{W}^{(T_{1})} -\mathbf{W}^{*}\|_{F}^{2}}{(2q-1) \eta(t - T_{1} + 1)} + \frac{\epsilon}{2q-1}.
\end{align*}
By setting \(t= T\), we find that
\begin{align*}
\frac{1}{T - T_{1} + 1}\sum_{s=T_{1}}^{T}L_{\mathcal{S}}(\mathbf{W}^{(s)}) \leq  \frac{\|\mathbf{W}^{(T_{1})} - \mathbf{W}^{*}\|_{F}^{2}}{(2q-1) \eta(T - T_{1} + 1)} + \frac{\epsilon}{2q-1} \leq \frac{3\epsilon}{2q-1} < \epsilon,
\end{align*}
where we utilize the condition that \(q> 2\) and the specific choice of \(T  = T_{1} + \Big\lfloor \frac{\|\mathbf{W}^{(T_{1})} - \mathbf{W}^{*}\|_{F}^{2}}{2\eta \epsilon}\Big\rfloor\). Since the mean value is less than \(\epsilon\), it follows that there must exist a time interval \(T_{1} \leq t \leq T\) for which \(L_{\mathcal{S}}(\mathbf{W}^{(t)}) < \epsilon\).
}

Finally, we aim to demonstrate that $\max_{j,r,i}|\rho_{j,r,i}^{(t)}| \leq 2\hat{\beta}$ holds 
 for all $ t \in [T_1, T]$. By inserting $T  = T_{1} + \Big\lfloor \frac{\|\mathbf{W}^{(T_{1})} - \mathbf{W}^{*}\|_{F}^{2}}{2\eta \epsilon}
\Big\rfloor$ into equation \eqref{eq:vanillasum}, we obtain 
\begin{align}
\sum_{s=T_{1}}^{T}L_{\mathcal{S}}(\mathbf{W}^{(s)}) &\leq \frac{2\|\mathbf{W}^{(T_{1})} - \mathbf{W}^{*}\|_{F}^{2}}{(2q-1) \eta}  = \tilde{O}(\eta^{-1}m^{3}\|\boldsymbol{\mu}\|_{2}^{2}), \label{eq: sum1}
\end{align}
where the inequality is a consequence of $\|\mathbf{W}^{(T_{1})} - \mathbf{W}^{*}\|_{F} \leq \tilde{O}(m^{3/2}\|\boldsymbol{\mu}\|_{2}^{-1})$ as shown in Lemma~\ref{lm:distance1}. 

Let's define $\Psi^{(t)} =  \max_{j,r,i}|\rho_{j,r,i}^{(t)}|$. We will employ induction to prove $\Psi^{(t)} \leq 2\hat{\beta}$ for all $ t \in [T_1, T]$. At $t = T_1$, by the definition of $\hat\beta$,  it is clear that $\Psi^{(T_1)} \leq \hat{\beta} \leq 2\hat{\beta}$. 

Assuming that there exists $\tilde{T} \in [T_1, T]$ such that $\Psi^{(t)} \leq 2\hat{\beta}$ for all $t \in [T_1, \tilde{T}-1]$, we can consider $t \in [T_1, \tilde{T}-1]$. Using the expression:
\begin{align}
    \rho_{j,r,i}^{(t+1)} &= \rho_{j,r,i}^{(t)} - \frac{\eta}{nm} \cdot \sum_{k \in \mathcal{N}(i)} D^{-1}_k \ell_k'^{(t)} \nonumber \\
     & \quad \sigma'\Bigg(\langle \mathbf{w}_{j,r}^{(0)}, \tilde{\boldsymbol{\xi}}_{k} \rangle + \sum_{ i'= 1 }^n \zeta_{j,r,i'}^{(t)} \frac{\langle \boldsymbol{\xi}_{i'}, \tilde{\boldsymbol{\xi}}_k \rangle }{\| \boldsymbol{\xi}_{i'} \|_2^2} + \sum_{ i' = 1}^n \omega_{j,r,i'}^{(t)} \frac{\langle \boldsymbol{\xi}_{i'}, \tilde{\boldsymbol{\xi}}_k \rangle}{\| \boldsymbol{\xi}_{i'} \|_2^2} \Bigg)\cdot \| \boldsymbol{\xi}_{i} \|_2^2  \label{eq:update_zeta2} 
\end{align}
we can proceed to analyze:
\begin{align*}
    \Psi^{(t+1)} &\leq \Psi^{(t)} + \max_{j,r,i}\bigg\{\frac{\eta }{nm} \cdot \sum_{k \in \mathcal{N}(i)} D^{-1}_k |\ell_k'^{(t)}| \cdot \sigma'\Bigg(\langle \mathbf{w}_{j,r}^{(0)}, \tilde{\boldsymbol{\xi}}_{k}\rangle + 2\sum_{ i'= 1 }^n \Psi^{(t)} \cdot \frac{ |\langle \boldsymbol{\xi}_{i'}, \tilde{\boldsymbol{\xi}}_k \rangle|}{ \| \boldsymbol{\xi}_{i'} \|_2^2}  \Bigg)\cdot \| \boldsymbol{\xi}_{i} \|_2^2 \bigg\} \\
    &= \Psi^{(t)} + \max_{j,r,i}\bigg\{\frac{\eta }{nm} \cdot \sum_{k \in \mathcal{N}(i)} D^{-1}_k |\ell_k'^{(t)}| \cdot \\
     & \quad \sigma'\Bigg(\langle \mathbf{w}_{j,r}^{(0)}, \tilde{\boldsymbol{\xi}}_{k}\rangle + 2\Psi^{(t)} + 2\sum_{ i'\neq k' }^n \Psi^{(t)} \cdot D^{-1}_k \sum_{k' \in \mathcal{N}(k)} \frac{ |\langle \boldsymbol{\xi}_{i'}, \boldsymbol{\xi}_{k'} \rangle|}{ \| \boldsymbol{\xi}_{i'} \|_2^2}  \Bigg)\cdot \| \boldsymbol{\xi}_{i} \|_2^2 \bigg\} \\
    &\leq \Psi^{(t)} + \frac{\eta q}{nm} \cdot \max_{i} \sum_{k \in \mathcal{N}(i)} D^{-1}_k | \ell_k'^{(t)}|\cdot \Bigg[2\cdot\sqrt{ \log(8mn/\delta)} \cdot \sigma_0 \sigma_p \sqrt{d/(n(p+s))} \\
    &\qquad+ \Bigg( 2+ \frac{4n \sigma_p^2 \cdot \sqrt{d \log(4n^2 / \delta)} }{ \sigma_p^2 d /2} \Bigg) \cdot \Psi^{(t)}  \Bigg]^{q-1}\cdot 2 \sigma_p^2 d\\
    &\leq \Psi^{(t)} + \frac{\eta q}{nm} \cdot \max_{i} \sum_{k \in \mathcal{N}(i)} D^{-1}_k | \ell_k'^{(t)}| \cdot \\
     & \quad \big(2\cdot \sqrt{ \log(8mn/\delta) } \cdot \sigma_0 \sigma_p \sqrt{d/(n(p+s))} + 4 \cdot \Psi^{(t)}  \big)^{q-1}\cdot 2 \sigma_p^2 d.
\end{align*}
The second inequality is derived from  Lemmas~\ref{lemma:data_innerproducts} and \ref{lemma:initialization_norms}, while the final inequality is based on the assumption that $d \geq 16n^2 \log (4n^2/\delta)$. By taking a telescoping sum, we can express the following:
\begin{align*}
    \Psi^{(T)} 
    &\overset{(i)}{\leq} \Psi^{(T_{1})} + \frac{\eta q}{nm} \sum_{s=T_{1}}^{\tilde{T}-1}\max_{i}   \sum_{k \in \mathcal{N}(i)} D^{-1}_k | \ell_k'^{(t)}|  \tilde{O}(\sigma_{p}^{2}d)\hat{\beta}^{q-1}\\
    &\overset{(ii)}{\leq} \Psi^{(T_{1})} + \frac{\eta q}{nm}\tilde{O}(\sigma_{p}^{2}d)\hat{\beta}^{q-1} \sum_{s=T_{1}}^{\tilde{T}-1}\max_{i} \sum_{k \in \mathcal{N}(i)} D^{-1}_k \ell_{k}^{(s)}\\
    &\overset{(iii)}{\leq} \Psi^{(T_{1})} + \tilde{O}(\eta m^{-1}\sigma_{p}^{2}d)\hat{\beta}^{q-1} \sum_{s=T_{1}}^{\tilde{T}-1}L_{\mathcal{S}}(\mathbf{W}^{(s)})\\
    &\overset{(iv)}{\leq} \Psi^{(T_{1})} + \tilde{O}(m^{2}\mathrm{SNR}^{-2})  \hat{\beta}^{q-1}\\
    &\overset{(v)}{\leq} \hat{\beta} + \tilde{O}(m^{2}n^{2/q} (n(p+s))^{1-2/q} \hat{\beta}^{q-2})\hat{\beta}\\
    &\overset{(vi)}{\leq} 2\hat{\beta}, 
\end{align*}
where (i) follows from our induction assumption that $\Psi^{(t)} \leq 2\hat{\beta}$, (ii) is derived from the relationship $|\ell'| \leq \ell$, (iii) is obtained by the sum of $\max_{i}\sum_{k \in \mathcal{N}(i)} D^{-1}_k \leq \sum_{i}\ell_{i}^{(s)} = n L_{\mathcal{S}}(\mathbf{W}^{(s)})$, (iv)  is due to the summation of $\sum_{s=T_{1}}^{\tilde{T} - 1}L_{S}(\mathbf{W}^{(s)}) \leq \sum_{s=T_{1}}^{T}L_{\mathcal{S}}(\mathbf{W}^{(s)})  = \tilde{O}(\eta^{-1}m^{3}\|\boldsymbol{\mu}\|_{2}^{2})$ as shown in \eqref{eq: sum1}, (v) is based on the condition $n \mathrm{SNR}^{q} \cdot  (n(p+s))^{q/2-1} \geq \tilde{\Omega}(1)$, and (vi) follows from the definition of $\hat{\beta} = \sigma_{0}\sigma_{p}\sqrt{d/(n(p+s))} / 2$ and $\tilde{O}(m^{2}n^{2/q} (n(p+s))^{1-2/q} \hat{\beta}^{q-2}) = \tilde{O}(m^{2}n^{2/q} (n(p+s))^{1-2/q}(\sigma_{0}\sigma_{p}\sqrt{d/(n(p+s))})^{q-2}) \leq 1$.

Thus, we conclude that $\Psi^{(\tilde{T})} \leq 2\hat{\beta}$, completing the induction and establishing the desired result.
\end{proof}

\section{Population loss} \label{seca:generalization}

Consider a new data point $(\mathbf{x},y)$ drawn from the SNM-SBM distribution. Without loss of generality, we suppose that the first patch is the signal patch and the second patch is the noise patch, i.e., $\mathbf{x} = [y \cdot \boldsymbol{\mu}, \boldsymbol{\xi}]$. Moreover, by the signal-noise decomposition, the learned neural network has parameter:
\begin{align*}
\mathbf{w}_{j,r}^{(t)} &= \mathbf{w}_{j,r}^{(0)} + j \cdot \gamma_{j,r}^{(t)} \cdot \frac{\boldsymbol{\mu}}{\|\boldsymbol{\mu}\|_{2}^{2}} + \sum_{ i = 1}^{n} \zeta_{j,r,i}^{(t)} \cdot \frac{\boldsymbol{\xi}_{i}}{\|\boldsymbol{\xi}_{i}\|_{2}^{2}} + \sum_{ i = 1}^{n} \omega_{j,r,i}^{(t)} \cdot \frac{\boldsymbol{\xi}_{i}}{\|\boldsymbol{\xi}_{i}\|_{2}^{2}}
\end{align*}
for $j\in\{\pm 1\}$ and $r\in [m]$.

{Although the framework of MLP diverges in certain nuances from the situation of MLP outlined in \cite{cao2022benign}, we are able to employ a similar analytical methodology to deduce the subsequent two lemmas:}

\begin{lemma}\label{lm:signalg2_appendix}
Under the same conditions as Theorem~\ref{thm:signal_learning_main}, we have that $\max_{j,r}|\langle \mathbf{w}_{j,r}^{(t)}, \tilde{\boldsymbol{\xi}}_{i} \rangle| \leq 1/2$ for all $0 \leq t \leq T$, and $i \in [n]$.
\end{lemma}

\begin{lemma}\label{lm:signalg}
Under the same conditions as Theorem~\ref{thm:signal_learning_main}, with probability at least $1 - 4mT \cdot \exp(-C_{2}^{-1}\sigma_{0}^{-2}\sigma_{p}^{-2}d^{-1} n(p+s) )$, we have that $\max_{j,r}|\langle \mathbf{w}_{j,r}^{(t)}, \tilde{\boldsymbol{\xi}} \rangle| \leq 1/2$ for all $0 \leq t \leq T$, where $C_{2} = \tilde{O}(1)$.
\end{lemma}

\begin{lemma}[Restatement of Lemma~\ref{lemma:signal_polulation_loss_main}]\label{lemma:signal_polulation_loss}
Let $T$ be defined in Lemma~\ref{lemma:phase1_main_sketch} respectively. Under the same conditions as Theorem~\ref{thm:signal_learning_main}, for any $0 \leq t \leq T$ with $L_{S}(\mathbf{W}^{(t)}) \leq 1$, it holds that $L_{\mathcal{D}}(\mathbf{W}^{(t)}) \leq c_1 \cdot L_{S}(\mathbf{W}^{(t)}) + \exp(- c_2 n^{2})$.
\end{lemma}

\begin{proof}[Proof of Lemma~\ref{lemma:signal_polulation_loss}]

Consider the occurrence of event \(\mathcal{E}\), defined as the condition under which Lemma~\ref{lm:signalg} is satisfied. We can then express the loss \(L_{\mathcal{D}}(\mathbf{W}^{(t)})\) as a sum of two components:

\begin{align}
\mathbb{E} \big[\ell\big(yf(\mathbf{W}^{(t)}, \tilde{\mathbf{x}})\big)\big] &= \underbrace{\mathbb{E}[\mathds{1}(\mathcal{E})\ell\big(yf(\mathbf{W}^{(t)}, \tilde{\mathbf{x}})\big)]}_{\text{Term } I_{1}} + \underbrace{\mathbb{E}[\mathds{1}(\mathcal{E}^{c})\ell\big(yf(\mathbf{W}^{(t)}, \tilde{\mathbf{x}})\big)]}_{\text{Term } I_{2}}.\label{eq:generalize_all}
\end{align}

Next, we proceed to establish bounds for \(I_1\) and \(I_2\).

\noindent\textbf{Bounding \(I_1\):} Given that \(L_{\mathcal{S}}(\mathbf{W}^{(t)}) \leq 1\), there must be an instance \((\tilde{\mathbf{x}}_{i},y_{i})\) for which \(\ell\big(y_{i}f(\mathbf{W}^{(t)}, \tilde{\mathbf{x}}_{i})\big)\leq L_\mathcal{S}(\mathbf{W}^{(t)}) \leq 1\), leading to \(y_{i}f(\mathbf{W}^{(t)}, \tilde{\mathbf{x}}_{i}) \geq 0\). Hence, we obtain:

\begin{align}
\exp(-y_{i}f(\mathbf{W}^{(t)}, \tilde{\mathbf{x}}_{i}))\overset{(i)}{\leq} 2\log\big(1+\exp(-y_{i}f(\mathbf{W}^{(t)},\tilde{\mathbf{x}}_{i}))\big) = 2\ell\big(y_{i}f(\mathbf{W}^{(t)}, \tilde{\mathbf{x}}_{i})\big) \leq 2L_\mathcal{S}(\mathbf{W}^{(t)}),\label{eq:I1_bound}    
\end{align}
where (i) follows from the inequality \(z \leq 2\log(1+z), \forall z \leq 1\). If event \(\mathcal{E}\) occurs, we deduce:

\begin{align}
|yf(\mathbf{W}^{(t)}, \tilde{\mathbf{x}}^{(2)}) - y_{i}f(\mathbf{W}^{(t)}, \tilde{\mathbf{x}}^{(2)}_{i})| \nonumber
&\leq \frac{1}{m}\sum_{j,r}\sigma(\langle\mathbf{w}_{j,r}, \tilde{\boldsymbol{\xi}}_{i}\rangle) + \frac{1}{m}\sum_{j,r}\sigma(\langle\mathbf{w}_{j,r}, \tilde{\boldsymbol{\xi}} \rangle ) \nonumber\\
&\leq 1.\label{eq:I1_bound2} 
\end{align}

Here, \(f(\mathbf{W}^{(t)}, \tilde{\mathbf{x}}^{(2)})\) refers to the input \(\tilde{\mathbf{x}}  = [0, \tilde{\mathbf{x}}^{(2)}]\). The second inequality is justified by Lemmas~\ref{lm:signalg} and \ref{lm:signalg2_appendix}. Consequently, we have:

\begin{align*}
I_{1}&\leq \mathbb{E}[\mathds{1}(\mathcal{E}) \exp(-yf(\mathbf{W}^{(t)}, \tilde{\mathbf{x}} ))] \\   
&=  \mathbb{E}[\mathds{1}(\mathcal{E}) \exp(-y_{i}f(\mathbf{W}^{(t)}, \tilde{\mathbf{x}}^{(1)})) \exp(-y_{i}f(\mathbf{W}^{(t)}, \tilde{\mathbf{x}}^{(2)}))] \\
&\leq 2e \cdot C \cdot \mathbb{E}[\mathds{1}(\mathcal{E}) \exp(-y_{i}f(\mathbf{W}^{(t)}, \tilde{\mathbf{x}}^{(1)}_{i})) \exp(-y_{i}f(\mathbf{W}^{(t)}, \tilde{\mathbf{x}}^{(2)}_{i}))] \\
&\leq 2e\cdot \mathbb{E}[\mathds{1}(\mathcal{E})  L_\mathcal{S}(\mathbf{W}^{(t)})],
\end{align*}

where the inequalities follow from the properties of cross-entropy loss, \eqref{eq:I1_bound2}, Lemma \ref{lemma:graph_numberofdata}, and \eqref{eq:I1_bound}. The constant \(c_1\) encapsulates the factors in the derivation.

\noindent\textbf{Estimating \(I_2\):} We now turn our attention to the second term \(I_{2}\). By selecting an arbitrary training data point \((\mathbf{x}_{i'},y_{i'})\) with \(y_{i'} = y\), we can derive the following:

\begin{align}
\ell\big(yf(\mathbf{W}^{(t)}, \tilde{\mathbf{x}})\big) &\leq \log(1 + \exp(F_{-y}(\mathbf{W}^{(t)}, \tilde{\mathbf{x}}))) \notag\\    
&\leq 1 + F_{-y}(\mathbf{W}^{(t)}, \tilde{\mathbf{x}} )\notag\\
&= 1 + \frac{1}{m}\sum_{j=-y,r\in [m]}\sigma(\langle\mathbf{w}_{j,r}^{(t)}, \tilde{y} \boldsymbol{\mu}\rangle) + \frac{1}{m}\sum_{j=-y,r\in [m]}\sigma(\langle\mathbf{w}_{j,r}^{(t)}, \tilde{\boldsymbol{\xi}} \rangle)\notag\\
&\leq 1 + F_{-y_{i}}(\mathbf{W}_{-y_{i'}}, \tilde{\mathbf{x}}_{i'}) + \frac{1}{m}\sum_{j=-y,r\in [m]}\sigma(\langle\mathbf{w}_{j,r}^{(t)}, \tilde{\boldsymbol{\xi}} \rangle) \notag\\
&\leq 2  + \frac{1}{m}\sum_{j=-y,r\in [m]}\sigma(\langle\mathbf{w}_{j,r}^{(t)}, \tilde{\boldsymbol{\xi}} \rangle)\notag\\
&\leq 2 + \tilde{O}((\sigma_{0}\sqrt{d})^{q})\| \tilde{\boldsymbol{\xi}} \|^{q}\label{eq:1.1},
\end{align}

where the inequalities follow from the properties of the cross-entropy loss and the constraints defined in Lemma~\ref{lm: F-yi}. The last inequality is a result of the boundedness of the inner product with \(\tilde{\boldsymbol{\xi}}\). Continuing, we have:

\begin{align*}
I_{2} &\leq   \sqrt{\mathbb{E}[\mathds{1}(\mathcal{E}^{c})]} \cdot \sqrt{\mathbb{E}\Big[\ell\big(yf(\mathbf{W}^{(t)}, \tilde{\mathbf{x}} )\big)^{2}\Big]}\\
&\leq \sqrt{\mathbb{P}(\mathcal{E}^{c})} \cdot \sqrt{4 + \tilde{O}((\sigma_{0}\sqrt{d})^{2q})\mathbb{E}[\| \tilde{\boldsymbol{\xi}} \|_{2}^{2q}]}\\
&\leq \exp \left[ -\tilde\Omega\left(\frac{\sigma_{0}^{-2}\sigma_{p}^{-2}}{d^{-1} n(p+s)}\right) +\text{polylog}(d) \right]\\
&\leq \exp(- c_1 n^{2}),
\end{align*}

where $c_1$ is a constant, the first inequality is by Cauchy-Schwartz inequality, the second inequality is by \eqref{eq:1.1}, the third inequality is by Lemma~\ref{lm:signalg} and the fact that $\sqrt{4 + \tilde{O}((\sigma_{0}\sqrt{d})^{2q})\mathbb{E}[\| \tilde{\boldsymbol{\xi}} \|_{2}^{2q}]} = O(\mathrm{poly}(d))$, and the last inequality is by our condition $\sigma_{0}\leq \tilde{O}(m^{-2/(q-2)}n^{-1})\cdot (\sigma_{p}\sqrt{d/(n(p+s))})^{-1}$ in Condition~\ref{condition:d_sigma0_eta}. Plugging the bounds of $I_{1}$, $I_{2}$ completes the proof.
\end{proof}

\section{Parallels between our data model and real-world dataset}

The citation network (Cora, Citeseer, and Pubmed) employ a bag-of-words feature representation, typically represented by one-hot vectors, thereby ensuring orthogonality between features. We can conceptually divide words into two categories: label-relevant and label-irrelevant. For example, words like ``algorithm" or ``neural network" are label-relevant to the subject of computer science, while general words like ``study" or ``approach" are label-irrelevant. In our SNM, $\boldsymbol{\mu}$ represents label-relevant features, while $\boldsymbol{\xi}$ represents label-irrelevant ones.

{Furthermore, the datasets Wiki-CS, Amazon-Computers, Amazon-Photo, Coauthor-CS, and Coauthor-Physics \cite{shchur2018pitfalls} also parallels with our theoretical model and we provide the more discussion as follows:}

\begin{itemize}

    \item  The Cora dataset includes 2,708 scientific publications, each categorized into one of seven classes, connected by 5,429 links. Each publication is represented by a binary word vector, which denotes the presence or absence of a corresponding word from a dictionary of 1,433 unique words.
    
   \item The Citeseer dataset comprises 3,312 scientific publications, each classified into one of six classes, connected by 4,732 links. Each publication is represented by a binary word vector, indicating the presence or absence of a corresponding word from a dictionary that includes 3,703 unique words.
    
    \item The Pubmed Diabetes dataset includes 19,717 scientific publications related to diabetes, drawn from the PubMed database and classified into one of three classes. The citation network is made up of 44,338 links. Each publication is represented by a TF-IDF weighted word vector from a dictionary consisting of 500 unique words.

    \item Coauthor CS (Computer Science) \& Coauthor Physics (Coauthor Phy.): The dataset typically includes features based on the keywords of an author's papers, and the task is often to predict each author's research field or interests based on their publication record and collaboration network.

    \item Amazon Computers \& Amazon Photo: Node features are derived from product reviews, and the classification task involves predicting product categories based on the co-purchase relationships and review data.

    \item WikiCS Node features could be derived from the text of the articles, such as word vectors. The classification task usually involves categorizing articles into different areas or subjects within Computer Science based on their content and the article network structure.
    
\end{itemize}

{We have broadened our analysis to include the measurement of cosine similarity between two equal-sized parts of node features (excluding the final feature for odd-sized representations) across a diverse range of datasets. This extended analysis bolsters the orthogonality relation posited in our model. The results are presented in Table \ref{tab:my_label}.}

\begin{table}[h]
\centering
\begin{tabular}{|l|c|r|}
\hline
\textbf{Dataset} & \textbf{Feature Dimension} & \textbf{Cossin Similarity} \\ \hline
Cora & 1433 & $1.57 \times 10^{-5}$ \\ \hline
Citeseer & 3703 & $3.99 \times 10^{-6}$ \\ \hline
Pubmed & 500 & $2.00 \times 10^{-4}$ \\ \hline
Coauthor CS & 6805 & $2.28 \times 10^{-6}$ \\ \hline
Coauthor Phy. & 8451 & $1.08 \times 10^{-6}$ \\ \hline
Amazon Comp. & 767 & $9.00 \times 10^{-4}$ \\ \hline
Amazon Photo & 745 & $9.00 \times 10^{-4}$ \\ \hline
WikiCS & 300 & $1.00 \times 10^{-4}$ \\ \hline
\end{tabular}
\caption{Cosine similarity analysis of node features across various datasets.}
\label{tab:my_label}
\end{table}

\section{Phase transition in GCN} \label{seca:phase}

In Figure \ref{fig:heatmap}, we illustrated the variance in test accuracy between MLP and GCN within a chosen range of SNR and sample numbers, where GCN was shown to achieve near-perfect test accuracy. Here, we broaden the SNR range towards the smaller end and display the corresponding phase diagram of GCN in Figure \ref{fig:gcn_phase}. When the SNR is exceedingly small, we observe that GCNs return lower test accuracy, suggesting the possibility of a phase transition in the test accuracy of GCNs.

\begin{figure}
    \centering
\includegraphics[width=0.65\textwidth]{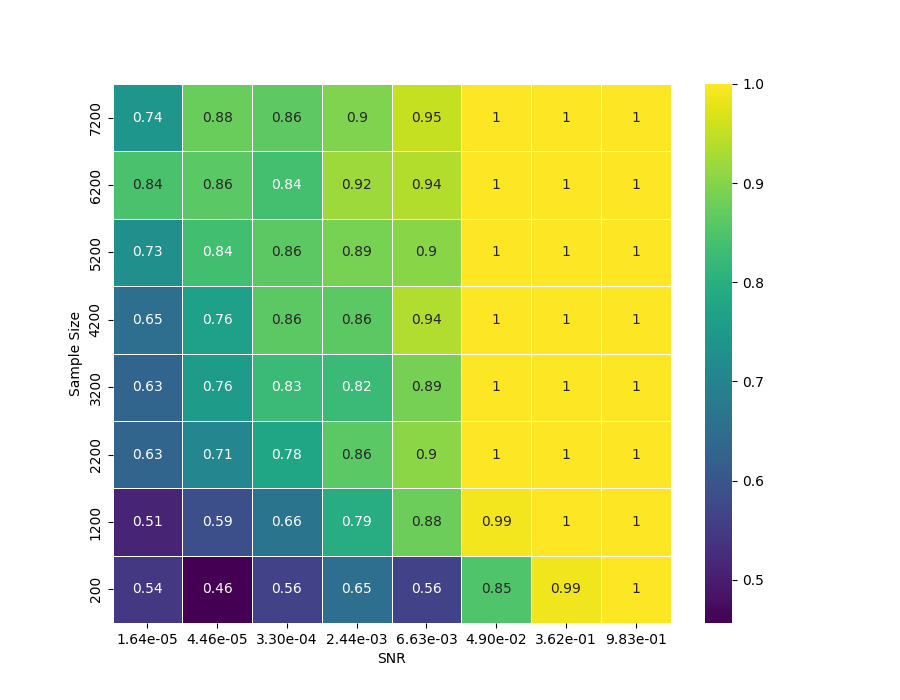}
    \caption{Test accuracy heatmap for GCNs after training.}
    \label{fig:gcn_phase}
\end{figure}

\begin{figure*}[h] 
\centering
\vspace{-2 pt}
\begin{minipage}{0.23\textwidth}
\includegraphics[width =1.3in]{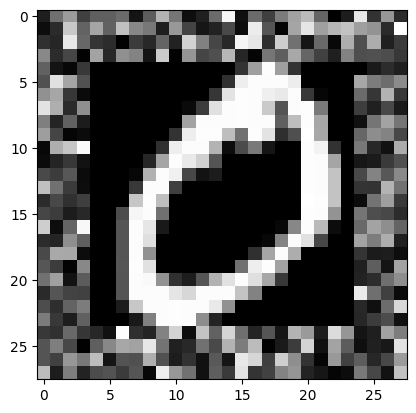}
\end{minipage} 
\hspace{2pt}
\begin{minipage}{0.23\textwidth}
\includegraphics[width =1.3in]{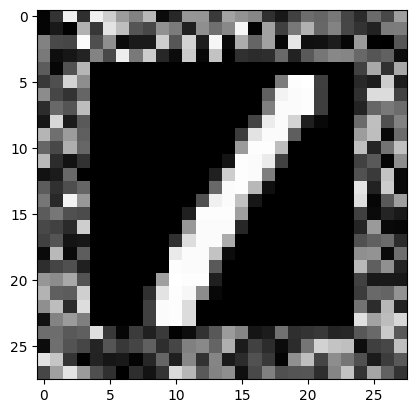}
\end{minipage}
\hspace{2pt}
\begin{minipage}{0.24\textwidth}
\includegraphics[width =1.38in]{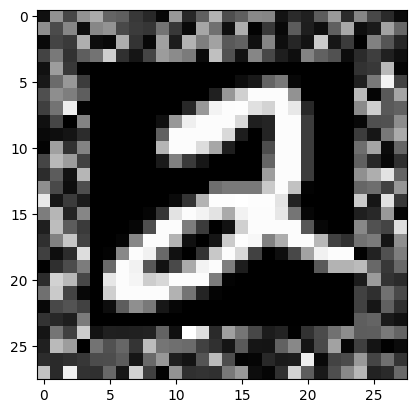}
\end{minipage}
\begin{minipage}{0.24\textwidth}
\includegraphics[width =1.39in]{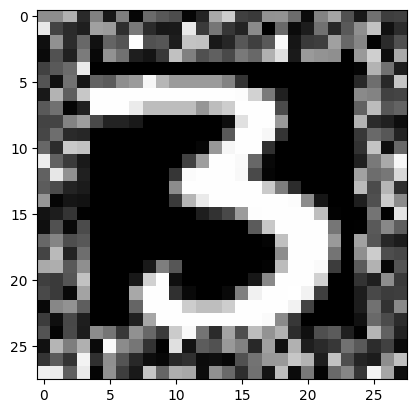}
\end{minipage}
\caption{Examples of adding noise patch to the numbers in the MNIST dataset.}
\label{fig:add_noise}
\end{figure*}

\section{How to Transform MNIST into a Signal-Noise Data Model} \label{seca:mnist}

To verify the theoretical study, we used the real-world data MNIST and modified it to align with the theoretical data model. In particular, we added Gaussian noise to the position of the image border while retaining the numbers in the middle. The final renderings are shown in Figure \ref{fig:add_noise}, where the noise level is chosen as $\sigma_p = 0.5$. A similar strategy can be found in \cite{cao2022benign}. As can be seen from Figure \ref{fig:add_noise}, the surrounding noise forms a patch, and the numbers in the middle form a signal patch. In subsequent experiments, we will separate noise patches and signal patches.

\section{Computation Resources} \label{sec:comp}

We implement our methods with PyTorch. For the software and hardware configurations, we
ensure the consistent environments for each datasets. We run all the experiments on Linux servers with  NVIDIA V100 graphics cards with CUDA 11.2.

\section{Broader Impacts} \label{sec:broad}

This work focuses on the theoretical understanding of the differences in optimization and generalization between GNNs and MLPs. The results established for GNNs may be applied to both theoretical and empirical research on GNNs. Additionally, we do not foresee any form of negative social impact induced by our work.

\end{document}